\documentclass{article}

\usepackage[utf8]{inputenc}
\usepackage[T1]{fontenc}
\usepackage{hyperref}
\usepackage{url}
\usepackage{booktabs}
\usepackage{amsfonts}
\usepackage{amsmath}
\usepackage{amssymb}
\usepackage{amsthm}
\usepackage{mathtools}
\usepackage{multirow}
\usepackage{nicefrac}
\usepackage{microtype}
\usepackage{graphicx}
\usepackage{xcolor}
\usepackage{subcaption}
\usepackage{algorithm}
\usepackage{algorithmic}
\usepackage{bm}
\usepackage{wrapfig}
\usepackage{tikz}
\usepackage{natbib}
\usepackage{float}
\usetikzlibrary{arrows,shapes,positioning,calc}

\newtheorem{theorem}{Theorem}
\newtheorem{proposition}{Proposition}

\newtheorem{corollary}{Corollary}
\newtheorem{definition}{Definition}
\newtheorem{remark}{Remark}

\newcommand{\R}{\mathbb{R}}

\newcommand{\eps}{\varepsilon}
\newcommand{\bx}{\bm{x}}
\newcommand{\ba}{\bm{a}}
\newcommand{\bb}{\bm{b}}
\newcommand{\bc}{\bm{c}}
\newcommand{\br}{\bm{r}}
\newcommand{\bmu}{\bm{\mu}}
\newcommand{\blam}{\bm{\lambda}}
\newcommand{\sign}{\operatorname{sign}}

\newcommand{\relu}{\operatorname{ReLU}}

\title{\textbf{Müntz-Szász Networks}: \\[0.3em] 
\large Neural Architectures with Learnable Power-Law Bases}

\author{
Gnankan Landry Regis N'guessan$^{1,2,3}$ \\[0.5em]
$^1$Axiom Research Group \\
$^2$Department of Applied Mathematics and Computational Science, \\
The Nelson Mandela African Institution of Science and Technology (NM-AIST), \\
Arusha, Tanzania \\
$^3$African Institute for Mathematical Sciences (AIMS), \\
Research and Innovation Centre (RIC), Kigali, Rwanda \\[0.5em]
\texttt{rnguessan@aimsric.org}
}

\date{}

\begin{document}

\maketitle

%%%%%%%%%%%%%%%%%%%%%%%%%%%%%%%%%%%%%%%%%%%%%%%%%%%%%%%%%%%%%%%%%%%%%%%%%%%%%%%
% ABSTRACT
%%%%%%%%%%%%%%%%%%%%%%%%%%%%%%%%%%%%%%%%%%%%%%%%%%%%%%%%%%%%%%%%%%%%%%%%%%%%%%%

\begin{abstract}
Standard neural network architectures employ fixed activation functions (ReLU, tanh, sigmoid) that are poorly suited for approximating functions with singular or fractional power behavior, a structure that arises ubiquitously in physics, including boundary layers, fracture mechanics, and corner singularities. We introduce \textbf{Müntz-Szász Networks (MSN)}, a novel architecture that replaces fixed smooth activations with learnable fractional power bases grounded in classical approximation theory. Each MSN edge computes $\phi(x) = \sum_k a_k |x|^{\mu_k} + \sum_k b_k \sign(x)|x|^{\lambda_k}$, where the exponents $\{\mu_k, \lambda_k\}$ are learned alongside the coefficients. We prove that MSN inherits universal approximation from the Müntz-Szász theorem and establish novel approximation rates: for functions of the form $|x|^\alpha$, MSN achieves error $\mathcal{O}(|\mu - \alpha|^2)$ with a \emph{single} learned exponent, whereas standard multilayer perceptrons (MLPs) require $\mathcal{O}(\epsilon^{-1/\alpha})$ neurons for comparable accuracy. On supervised regression with singular target functions, MSN achieves \textbf{5-8$\times$ lower error} than MLPs with \textbf{10$\times$ fewer parameters}. Physics-informed neural networks (PINNs) represent a particularly demanding application for singular function approximation; on PINN benchmarks including a singular ordinary differential equation (ODE) and stiff boundary-layer problems, MSN achieves \textbf{3-6$\times$ improvement} while learning interpretable exponents that match the known solution structure. Our results demonstrate that theory-guided architectural design can yield dramatic improvements for scientifically-motivated function classes.
\end{abstract}

%%%%%%%%%%%%%%%%%%%%%%%%%%%%%%%%%%%%%%%%%%%%%%%%%%%%%%%%%%%%%%%%%%%%%%%%%%%%%%%
% KEYWORDS
%%%%%%%%%%%%%%%%%%%%%%%%%%%%%%%%%%%%%%%%%%%%%%%%%%%%%%%%%%%%%%%%%%%%%%%%%%%%%%%

\vspace{0.5em}

%%%%%%%%%%%%%%%%%%%%%%%%%%%%%%%%%%%%%%%%%%%%%%%%%%%%%%%%%%%%%%%%%%%%%%%%%%%%%%%
% 1. INTRODUCTION
%%%%%%%%%%%%%%%%%%%%%%%%%%%%%%%%%%%%%%%%%%%%%%%%%%%%%%%%%%%%%%%%%%%%%%%%%%%%%%%

\section{Introduction}
\label{sec:introduction}

Neural networks have achieved remarkable success across diverse domains, yet their effectiveness depends critically on matching architectural inductive biases to problem structure. Standard architectures employ smooth activation functions (tanh, sigmoid) or piecewise-linear functions (ReLU), which are universal approximators but poorly suited for functions with singular or fractional power behavior. Such functions arise ubiquitously in physics: boundary layers exhibit $\mathcal{O}(\epsilon^{1/2})$ thickness scaling \citep{schlichting2016boundary}, crack-tip stresses scale as $r^{-1/2}$ in fracture mechanics \citep{anderson2005fracture}, and corner singularities in elliptic partial differential equations (PDEs) produce solutions of the form $r^{\pi/\omega}$ where $\omega$ is the corner angle \citep{grisvard1992elliptic}.

Physics-informed neural networks (PINNs) \citep{raissi2019physics} embed physical laws directly into the learning objective, enabling mesh-free solutions to differential equations. However, PINNs inherit the approximation limitations of their underlying architectures. For a function $f(x) = |x|^\alpha$ with $0 < \alpha < 1$, achieving approximation error $\epsilon$ requires $\mathcal{O}(\epsilon^{-1/\alpha})$ ReLU neurons \citep{yarotsky2017error}, an exponential cost in the singularity strength. This limitation manifests in practice as PINN failures on stiff boundary-layer problems \citep{krishnapriyan2021characterizing} and high errors near singularities.

We propose a fundamentally different approach: rather than approximating singular functions with smooth or piecewise-linear bases, we \emph{learn the exponents directly}. Our architecture is grounded in the classical Müntz-Szász theorem \citep{muntz1914, szasz1916}, which characterizes when power functions with arbitrary real exponents form a complete basis:

\begin{theorem}[Müntz-Szász, 1914]
\label{thm:muntz-intro}
The span of $\{1, x^{\lambda_1}, x^{\lambda_2}, \ldots\}$ with $0 < \lambda_1 < \lambda_2 < \cdots$ is dense in $C[0,1]$ if and only if $\sum_{k=1}^\infty 1/\lambda_k = \infty$.
\end{theorem}

This theorem guarantees that with appropriately chosen exponents, power functions can approximate \emph{any} continuous function, including those with fractional power structure. Our \textbf{Müntz-Szász Networks (MSN)} leverage this insight by learning the exponents $\{\mu_k, \lambda_k\}$ alongside the coefficients, enabling the network to discover and exploit the natural power structure of the target function. MSN is a general-purpose approximation architecture for functions with power-law structure; PINNs represent a particularly demanding application where this structure commonly arises.

To address these limitations, we introduce Müntz-Szász Networks and make the following contributions:

\begin{enumerate}
    \item \textbf{Architecture:} We introduce MSN, a neural network architecture where each edge computes a Müntz polynomial with learnable exponents. We develop a bounded parameterization ensuring numerical stability and propose a Müntz divergence regularizer connecting to classical approximation theory (Section~\ref{sec:method}).
    
    \item \textbf{Theory:} We prove universal approximation for MSN (Theorem~\ref{thm:uat}) and establish novel approximation rates showing MSN achieves $\mathcal{O}(\delta^2)$ error for power functions when the learned exponent is within $\delta$ of the true exponent (Theorem~\ref{thm:approx-rate}), compared to $\mathcal{O}(N^{-2\alpha})$ for MLPs (Section~\ref{sec:theory}).
    
    \item \textbf{Experiments:} On supervised regression with singular functions ($\sqrt{x}$, $|x-0.5|^{0.2}$), MSN achieves 5-8$\times$ lower error than MLPs with 10$\times$ fewer parameters. On PINN benchmarks (singular ODE, boundary-layer boundary value problems), MSN achieves 3-6$\times$ improvement while learning interpretable exponents (Section~\ref{sec:experiments}).
\end{enumerate}

%%%%%%%%%%%%%%%%%%%%%%%%%%%%%%%%%%%%%%%%%%%%%%%%%%%%%%%%%%%%%%%%%%%%%%%%%%%%%%%
% 2. BACKGROUND
%%%%%%%%%%%%%%%%%%%%%%%%%%%%%%%%%%%%%%%%%%%%%%%%%%%%%%%%%%%%%%%%%%%%%%%%%%%%%%%

\section{Background}
\label{sec:background}

\subsection{The Müntz-Szász Theorem}

% REVISION: Added historical positioning (2-3 sentences)
The Müntz-Szász theorem, originating in the work of Müntz \citep{muntz1914} and Szász \citep{szasz1916}, provides a complete characterization of when systems of monomials with non-integer exponents form a dense subset of $C[0,1]$. Together with the Weierstrass approximation theorem, it stands as a cornerstone of constructive approximation theory, addressing the fundamental question of which function systems can represent arbitrary continuous functions. For comprehensive treatments, see \citet{borwein1995polynomials} and \citet{lorentz1996constructive}.

\begin{definition}[Müntz System]
For a sequence $\Lambda = \{0 \leq \lambda_0 < \lambda_1 < \cdots\}$, the \emph{Müntz system} is $\mathcal{M}_\Lambda = \{x^{\lambda_k}\}_{k=0}^\infty$.
\end{definition}

The divergence condition $\sum_{k=1}^\infty 1/\lambda_k = \infty$ in Theorem~\ref{thm:muntz-intro} ensures that exponents do not grow too rapidly. For example, $\Lambda = \{0, 1, 2, 3, \ldots\}$ (polynomials) trivially satisfies the condition, as does $\Lambda = \{0, 0.5, 1, 1.5, 2, \ldots\}$. However, $\Lambda = \{0, 2, 4, 8, 16, \ldots\}$ violates it, and indeed the span of $\{1, x^2, x^4, x^8, \ldots\}$ cannot approximate $x$ on $[0,1]$.

For symmetric intervals, we require the \emph{full Müntz theorem}:

\begin{theorem}[Full Müntz Theorem]
\label{thm:full-muntz}
Let $0 < \lambda_1 < \lambda_2 < \cdots$ with $\sum_k 1/\lambda_k = \infty$. Then the system 
\[
\{1\} \cup \{|x|^{\lambda_k}\}_{k=1}^\infty \cup \{\sign(x)|x|^{\lambda_k}\}_{k=1}^\infty
\]
is dense in $C[-1,1]$.
\end{theorem}

The constant function $1$ is essential for approximating functions $f$ with $f(0) \neq 0$, since $|0|^{\lambda} = 0$ for all $\lambda > 0$. The even functions $|x|^{\lambda}$ and odd functions $\sign(x)|x|^{\lambda}$ together decompose any continuous function on symmetric domains.

% REVISION: Added explicit disclaimer about not using full theorem setting
\begin{remark}[Relationship to Classical Theory]
\label{rem:classical}
We emphasize that MSNs do not implement the Müntz-Szász theorem in its classical form. The theorem concerns uniform density of \emph{infinite} monomial systems in $C[0,1]$, whereas MSNs operate with \emph{finite}, learned exponent sets, empirical loss functions (typically $L^2$), and compositional architectures. In this sense, the theorem serves as a \emph{guiding principle} rather than a direct blueprint: it motivates the use of learnable power bases and informs the design of our regularizer, but the network itself is a finite, trainable approximation. Empirical risk minimization and compositional architectures are necessary to scale beyond classical approximation settings and to integrate with modern learning pipelines.
\end{remark}

\subsection{Physics-Informed Neural Networks}

PINNs \citep{raissi2019physics} solve differential equations by training neural networks to satisfy physical laws. For a PDE $\mathcal{N}[u](\bx) = 0$ on domain $\Omega$ with boundary conditions $\mathcal{B}[u](\bx) = g(\bx)$ on $\partial\Omega$, the PINN loss is:
\begin{equation}
\label{eq:pinn-loss}
\mathcal{L} = \frac{1}{N_c}\sum_{i=1}^{N_c} |\mathcal{N}[u_\theta](\bx_i)|^2 + \lambda_{BC} \frac{1}{N_b}\sum_{j=1}^{N_b} |\mathcal{B}[u_\theta](\bx_j) - g(\bx_j)|^2,
\end{equation}
where derivatives are computed via automatic differentiation.

PINNs face several challenges for singular solutions:
\begin{itemize}
    \item \textbf{Spectral bias:} Networks learn low-frequency components before high-frequency ones \citep{rahaman2019spectral, wang2021eigenvector}, causing slow convergence near singularities.
    \item \textbf{Gradient pathologies:} Imbalanced loss terms cause unstable training \citep{wang2021understanding}.
    \item \textbf{Representation limits:} Smooth activations cannot efficiently represent fractional powers \citep{yarotsky2017error}.
\end{itemize}

MSN addresses the third challenge directly by learning appropriate power bases, potentially alleviating the first two as well.

%%%%%%%%%%%%%%%%%%%%%%%%%%%%%%%%%%%%%%%%%%%%%%%%%%%%%%%%%%%%%%%%%%%%%%%%%%%%%%%
% 3. METHOD
%%%%%%%%%%%%%%%%%%%%%%%%%%%%%%%%%%%%%%%%%%%%%%%%%%%%%%%%%%%%%%%%%%%%%%%%%%%%%%%

\section{Müntz-Szász Networks}
\label{sec:method}

We now introduce Müntz-Szász Networks (MSN), beginning with the basic building block and progressing to the full architecture.

\subsection{Architecture}

% FIGURE 1: Architecture diagram
\begin{figure}[t]
    \centering
    \includegraphics[width=\textwidth]{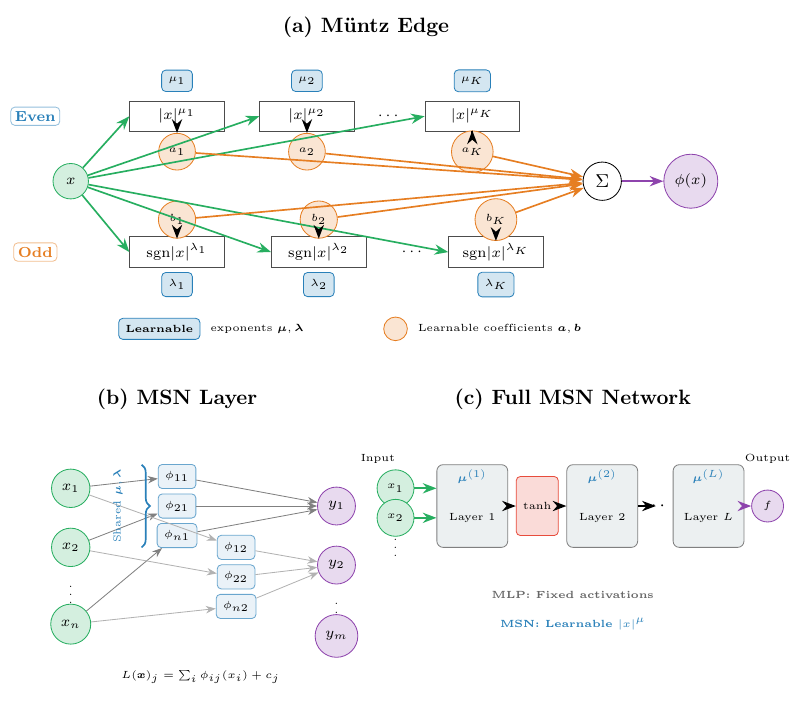}
    \caption{\textbf{Müntz-Szász Network architecture.} (a) A Müntz edge computes $\phi(x) = \sum_k a_k |x|^{\mu_k} + \sum_k b_k \sign(x)|x|^{\lambda_k}$ with learnable exponents $\bmu, \blam$ and coefficients $\ba, \bb$. (b) An MSN layer connects all input-output pairs via Müntz edges with shared exponents. (c) Full MSN stacks layers with inter-layer nonlinearities ($\tanh$). Unlike standard networks with fixed activations, MSN learns the functional form of each edge.}
    \label{fig:architecture}
\end{figure}

\begin{figure}[!htbp]
    \centering
    \includegraphics[width=\textwidth]{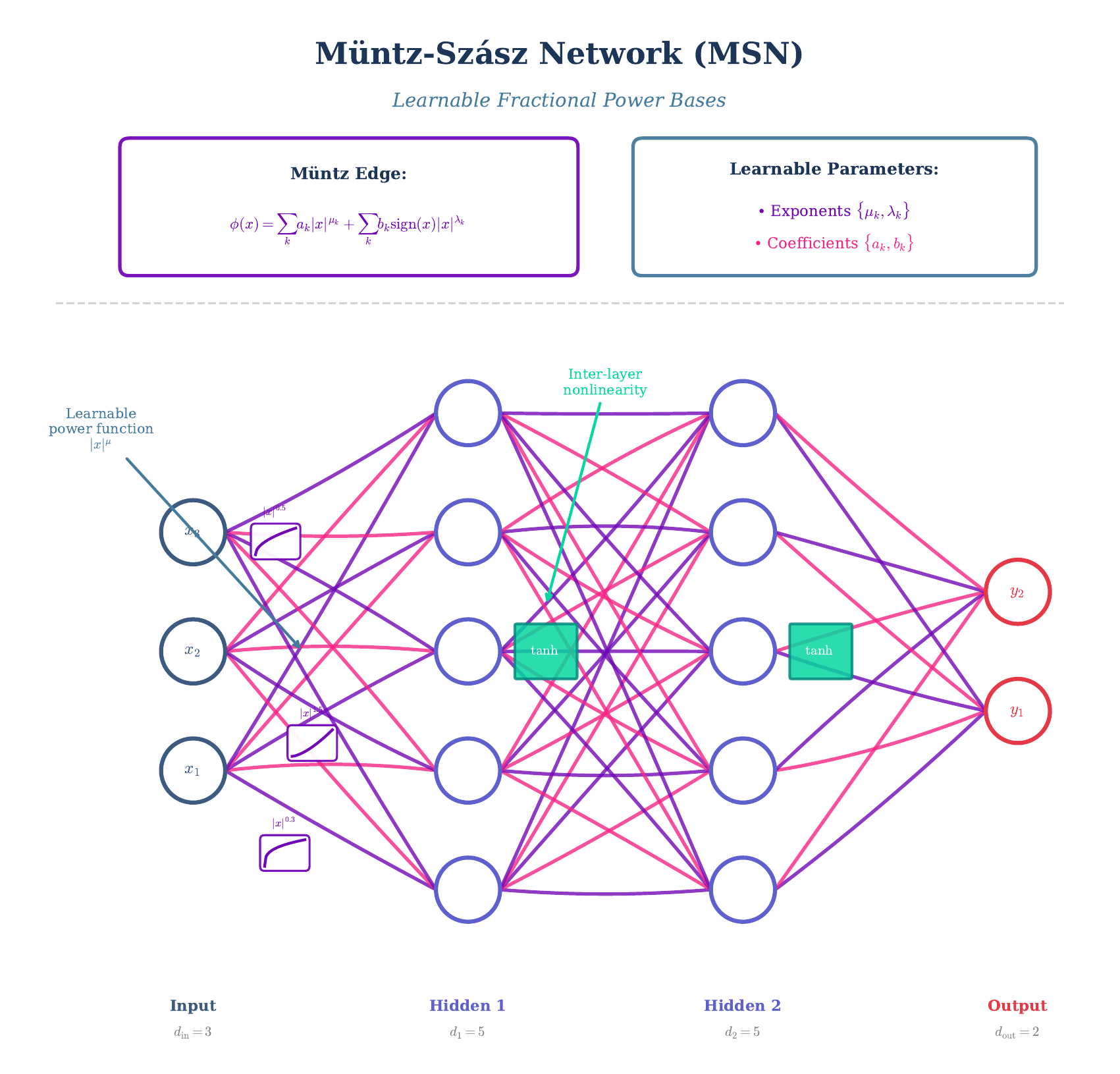}
    \caption{\textbf{Full MSN architecture with learnable power functions.} The network transforms inputs through MSN layers connected by $\tanh$ nonlinearities. Small inset plots on representative edges visualize the learned power functions (e.g., $|x|^{0.3}$, $|x|^{0.5}$, $|x|^{1.5}$). This architecture enables MSN to discover and exploit the natural power structure of target functions.}
    \label{fig:architecture-full}
\end{figure}

\begin{figure}[!htbp]
    \centering
    \includegraphics[width=\textwidth]{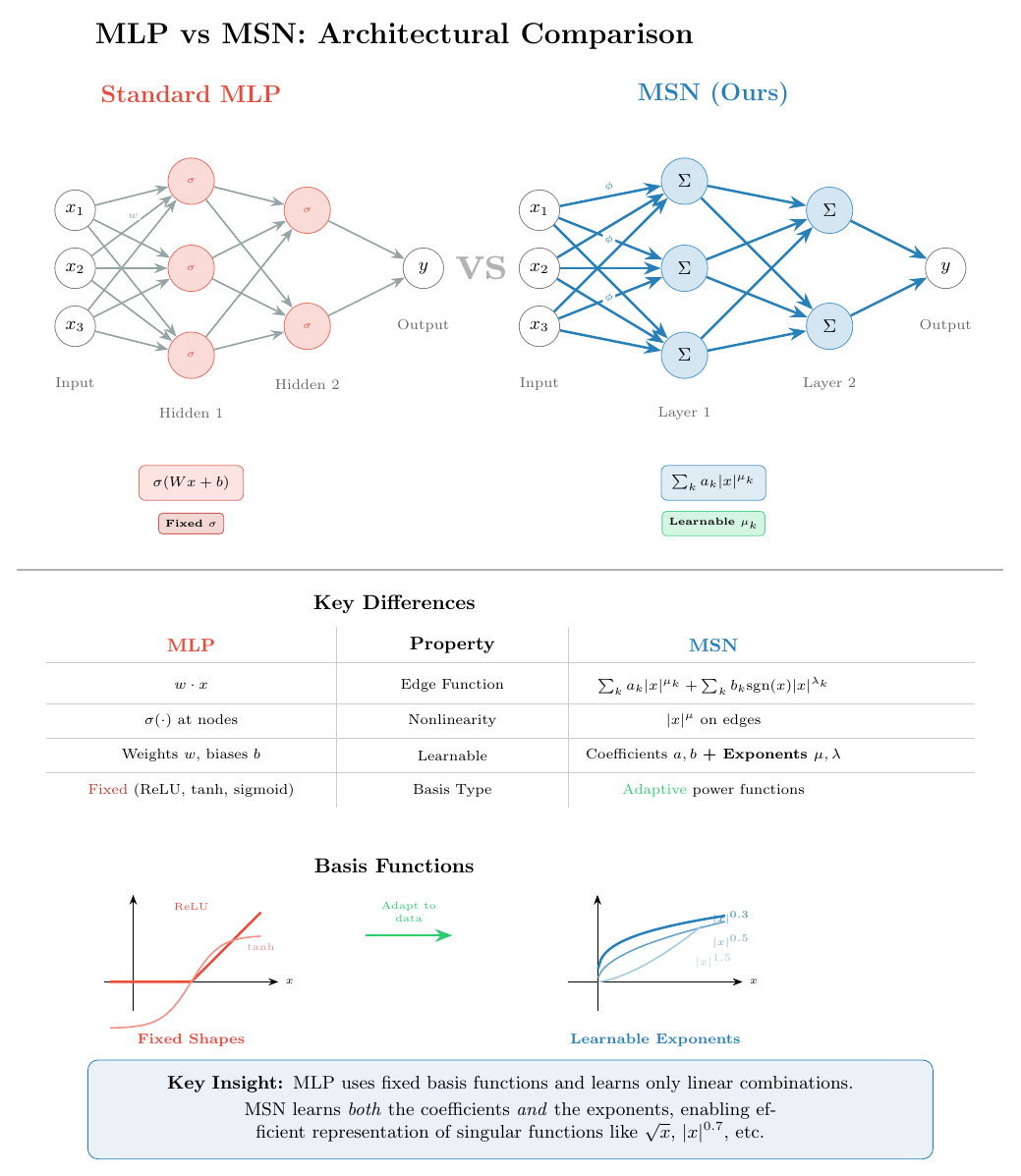}
    \caption{\textbf{Comparison of MLP and MSN architectures.} 
    MLP uses fixed activation functions (ReLU, tanh) at nodes and learns 
    only weights on edges. MSN uses learnable power functions $|x|^{\mu_k}$ 
    on edges, learning both coefficients and exponents. This enables 
    efficient representation of singular functions.}
    \label{fig:comparison}
\end{figure}

\begin{definition}[Müntz Edge]
\label{def:muntz-edge}
A \emph{Müntz edge} $\phi: \R \to \R$ is defined as
\begin{equation}
\label{eq:muntz-edge}
\phi(x; \ba, \bb, \bmu, \blam) = \sum_{k=1}^{K_e} a_k |x|^{\mu_k} + \sum_{k=1}^{K_o} b_k \sign(x) |x|^{\lambda_k},
\end{equation}
where $\bmu = (\mu_1, \ldots, \mu_{K_e}) \in \R_{>0}^{K_e}$ are \emph{even exponents}, $\blam = (\lambda_1, \ldots, \lambda_{K_o}) \in \R_{>0}^{K_o}$ are \emph{odd exponents}, and $\ba, \bb$ are learnable coefficients.
\end{definition}

The decomposition into even ($|x|^\mu$) and odd ($\sign(x)|x|^\lambda$) terms mirrors the parity decomposition of functions on symmetric domains and is justified by Theorem~\ref{thm:full-muntz}.

\begin{definition}[MSN Layer]
\label{def:msn-layer}
An \emph{MSN layer} $L: \R^{d_{\text{in}}} \to \R^{d_{\text{out}}}$ is:
\begin{equation}
\label{eq:msn-layer}
L(\bx)_j = \sum_{i=1}^{d_{\text{in}}} \phi_{ij}(x_i) + c_j, \quad j = 1, \ldots, d_{\text{out}},
\end{equation}
where each $\phi_{ij}$ is a Müntz edge and $c_j$ are bias terms.
\end{definition}

\begin{definition}[Müntz-Szász Network]
\label{def:msn}
A \emph{Müntz-Szász Network} with $L$ layers is:
\begin{equation}
\label{eq:msn}
f_{\text{MSN}}(\bx) = L_L \circ \tau \circ L_{L-1} \circ \tau \circ \cdots \circ \tau \circ L_1(\bx),
\end{equation}
where $\tau = \tanh$ is applied elementwise between layers.
\end{definition}

\paragraph{Exponent Sharing.} In practice, we share exponents $(\bmu, \blam)$ across all edges within a layer, reducing exponent parameters from $\mathcal{O}(d_{\text{in}} \cdot d_{\text{out}} \cdot K)$ to $\mathcal{O}(K)$ per layer while maintaining distinct coefficients $(\ba, \bb)$ per edge.

\subsection{Exponent Parameterization}
\label{sec:exponent-param}

Directly optimizing exponents in $\R_{>0}$ is numerically unstable: gradients can push exponents negative (undefined) or extremely large (overflow for $|x| > 1$). We propose a \emph{bounded parameterization} ensuring well-behaved exponents.

\begin{definition}[Bounded Exponent Map]
\label{def:bounded-exp}
For raw parameters $\br \in \R^K$, maximum exponent $p_{\max} > 0$, and margin $\eps > 0$:
\begin{equation}
\label{eq:bounded-exp}
\mu_k = \eps + (p_{\max} - 2\eps) \cdot \text{sort}(\sigma(\br))_k,
\end{equation}
where $\sigma(r) = 1/(1 + e^{-r})$ is the sigmoid function.
\end{definition}

\begin{proposition}[Properties of Bounded Parameterization]
\label{prop:bounded-properties}
The map $\br \mapsto \bmu$ satisfies:
\begin{enumerate}
    \item \textbf{Strict positivity:} $\mu_k \in (\eps, p_{\max} - \eps)$ for all $k$.
    \item \textbf{Ordering:} $\mu_1 < \mu_2 < \cdots < \mu_K$ almost surely.
    \item \textbf{Differentiability:} The map is differentiable almost everywhere.
    \item \textbf{Gradient bound:} $\|\partial \bmu / \partial \br\|_\infty \leq p_{\max}/4$.
\end{enumerate}
\end{proposition}

% REVISION: Fixed wording about ordering and Müntz condition
The proof is given in Appendix~\ref{app:proof-bounded}. The ordering property preserves a canonical Müntz-like structure and facilitates enforcing the divergence condition via regularization, while the gradient bound prevents instabilities during optimization.

\subsection{Müntz Divergence Regularizer}
\label{sec:regularizer}

The Müntz condition $\sum_k 1/\lambda_k = \infty$ ensures basis completeness. For finite $K$, we encourage configurations that would satisfy this condition asymptotically.

\begin{definition}[Müntz Divergence]
\label{def:muntz-div}
For exponents $\bmu \in \R_{>0}^{K_e}$ and $\blam \in \R_{>0}^{K_o}$:
\begin{equation}
\label{eq:muntz-div}
D(\bmu, \blam) = \sum_{k=1}^{K_e} \frac{1}{\mu_k} + \sum_{k=1}^{K_o} \frac{1}{\lambda_k}.
\end{equation}
\end{definition}

\begin{definition}[Müntz Regularizer]
\label{def:muntz-reg}
For threshold $C > 0$:
\begin{equation}
\label{eq:muntz-reg}
\mathcal{R}_{\text{Müntz}}(\bmu, \blam) = \relu(C - D(\bmu, \blam)).
\end{equation}
\end{definition}

\begin{proposition}[Regularizer Effect]
\label{prop:reg-effect}
The regularizer $\mathcal{R}_{\text{Müntz}}$ encourages:
\begin{enumerate}
    \item \textbf{Small exponents:} To maximize $D$, some $\mu_k$ should be small (near zero).
    \item \textbf{Diverse exponents:} Spreading exponents increases $D$ compared to concentration at a single value.
\end{enumerate}
\end{proposition}

Both effects align with the Müntz condition's requirement for diverse, slowly-growing exponents. We set $C = 2$ in experiments, which for $K = 6$ exponents requires an average exponent value below $3.0$, sufficient to encourage at least one small exponent near zero while allowing others to capture higher-order behavior. Preliminary experiments showed robust performance for $C \in [1, 5]$.

% REVISION: Added clarification that regularizer is not required for UAT
\begin{remark}
This regularizer is not required for universal approximation but improves optimization stability and encourages diverse exponent configurations in finite networks. It serves as a soft constraint guiding the learned exponents toward configurations consistent with the classical theory.
\end{remark}

\subsection{Training Objective and Stabilization}
\label{sec:training}

\paragraph{Training Objective.} For PINNs:
\begin{equation}
\label{eq:total-loss}
\mathcal{L} = \mathcal{L}_{\text{PDE}} + \lambda_{\text{BC}} \mathcal{L}_{\text{BC}} + \beta_1 \mathcal{R}_{\text{Müntz}} + \beta_2 (\|\ba\|_1 + \|\bb\|_1),
\end{equation}
where the $L^1$ penalty promotes coefficient sparsity. For supervised regression, we use $\mathcal{L}_{\text{MSE}}$ in place of $\mathcal{L}_{\text{PDE}} + \lambda_{\text{BC}}\mathcal{L}_{\text{BC}}$.

\paragraph{Training Stabilization.} Learning exponents introduces optimization challenges. We employ three techniques:

\begin{enumerate}
    \item \textbf{Two-time-scale optimization:} Exponents evolve slowly ($\eta_{\text{exp}} = 0.02\eta$) while coefficients adapt faster ($\eta_{\text{coeff}} = \eta$).
    
    \item \textbf{Exponent warmup:} During the first $T_{\text{warm}}$ steps, exponents are frozen to allow coefficients to find a reasonable configuration.
    
    \item \textbf{Exponent gradient clipping:} We clip exponent gradients separately with a small threshold $\delta \in [0.03, 0.1]$.
\end{enumerate}

Algorithm~\ref{alg:msn-training} summarizes the training procedure.

\begin{algorithm}[t]
\caption{MSN Training for PINNs}
\label{alg:msn-training}
\begin{algorithmic}[1]
\REQUIRE PDE operator $\mathcal{N}$, BC operator $\mathcal{B}$, collocation points $\{x_i^c\}$, boundary points $\{x_j^b\}$
\REQUIRE Learning rate $\eta$, warmup steps $T_{\text{warm}}$, exponent clip $\delta$
\STATE Initialize MSN $f_\theta$ with random coefficients $(\ba, \bb)$ and exponents $(\br_e, \br_o)$
\FOR{$t = 1, \ldots, T$}
    \STATE Compute exponents: $\bmu \gets \text{BoundedMap}(\br_e)$, $\blam \gets \text{BoundedMap}(\br_o)$
    \STATE Compute losses: $\mathcal{L}_{\text{PDE}}, \mathcal{L}_{\text{BC}}, \mathcal{R}_{\text{Müntz}}$
    \STATE $\mathcal{L} \gets \mathcal{L}_{\text{PDE}} + \lambda_{\text{BC}}\mathcal{L}_{\text{BC}} + \beta_1 \mathcal{R}_{\text{Müntz}} + \beta_2(\|\ba\|_1 + \|\bb\|_1)$
    \STATE Compute gradients $\nabla_\theta \mathcal{L}$
    \IF{$t > T_{\text{warm}}$}
        \STATE Clip exponent gradients: $\nabla_{\br_e, \br_o} \mathcal{L} \gets \text{ClipNorm}(\nabla_{\br_e, \br_o} \mathcal{L}, \delta)$
        \STATE Update exponents: $\br_e \gets \br_e - 0.02\eta \cdot \nabla_{\br_e}\mathcal{L}$, $\br_o \gets \br_o - 0.02\eta \cdot \nabla_{\br_o}\mathcal{L}$
    \ENDIF
    \STATE Update coefficients: $(\ba, \bb, \bc) \gets (\ba, \bb, \bc) - \eta \cdot \nabla_{\ba, \bb, \bc}\mathcal{L}$
\ENDFOR
\RETURN Trained MSN $f_\theta$
\end{algorithmic}
\end{algorithm}

%%%%%%%%%%%%%%%%%%%%%%%%%%%%%%%%%%%%%%%%%%%%%%%%%%%%%%%%%%%%%%%%%%%%%%%%%%%%%%%
% 4. THEORETICAL ANALYSIS
%%%%%%%%%%%%%%%%%%%%%%%%%%%%%%%%%%%%%%%%%%%%%%%%%%%%%%%%%%%%%%%%%%%%%%%%%%%%%%%

\section{Theoretical Analysis}
\label{sec:theory}

We establish universal approximation and derive approximation rates demonstrating MSN's advantages for singular functions.

\subsection{Universal Approximation}

\begin{theorem}[Universal Approximation]
\label{thm:uat}
Let $f \in C([-1,1]^d)$ and $\eps > 0$. There exists an MSN $f_{\text{MSN}}$ with finite depth, finite width, and exponents satisfying the Müntz condition such that $\|f - f_{\text{MSN}}\|_\infty < \eps$.
\end{theorem}

\begin{proof}[Proof Sketch]
By Theorem~\ref{thm:full-muntz}, Müntz systems (including the constant function) are dense in $C[-1,1]$. Tensor products of dense univariate systems are dense in $C([-1,1]^d)$. MSN layers compute sums of univariate Müntz functions with bias terms providing the constant, and the inter-layer nonlinearity $\tau = \tanh$ provides additional expressivity. Standard UAT arguments \citep{cybenko1989approximation, hornik1991approximation} complete the proof. Full details in Appendix~\ref{app:proof-uat}.
\end{proof}

% REVISION: Added remark clarifying the role of tanh
\begin{remark}[Role of Inter-Layer Nonlinearity]
\label{rem:tanh-role}
While a single MSN layer already defines a rich approximation space via Müntz polynomials, inter-layer nonlinearities ($\tanh$) ensure closure under composition and simplify multivariate approximation. This design aligns MSN with standard neural network architectures while preserving the power-law expressivity of Müntz bases within each layer.
\end{remark}

\subsection{Approximation Rates for Power Functions}

The key advantage of MSN is not merely universal approximation (MLPs also satisfy this) but rather the \emph{rate} of approximation for singular function classes.

\begin{proposition}[MLP Approximation Limits]
\label{prop:mlp-rate}
For $f(x) = |x|^\alpha$ with $0 < \alpha < 1$:
\begin{itemize}
    \item Networks with smooth activations (tanh, sigmoid) require $\Omega(N^{1/\alpha})$ neurons to achieve $\mathcal{O}(1/N)$ error.
    \item ReLU networks with $P$ linear pieces achieve error $\mathcal{O}(P^{-2})$ with matching lower bound $\Omega(P^{-2})$.
\end{itemize}
\end{proposition}

\begin{remark}
The ReLU bound is stated in terms of the number of linear pieces $P$, not the number of neurons $N$. For deep ReLU networks, $P$ can grow exponentially in depth while $N$ grows linearly, so bounds written solely in terms of neurons can be misleading. The $\mathcal{O}(P^{-2})$ rate reflects the fundamental limitation of piecewise-linear approximation.
\end{remark}

This follows from \citet{yarotsky2017error} and \citet{devore2021neural}. The limitation arises because smooth and piecewise-linear functions cannot efficiently represent the cusp at $x = 0$.

In contrast, MSN achieves dramatically better rates:

\begin{theorem}[MSN Approximation Rate]
\label{thm:approx-rate}
Let $f(x) = x^\alpha$ for $\alpha > -1/2$. For an MSN edge with even exponents $\bmu = (\mu_1, \ldots, \mu_K)$ where $\mu_k > -1/2$ are distinct:
\begin{equation}
\label{eq:approx-rate}
\inf_{\ba \in \R^K} \left\| f - \sum_{k=1}^K a_k x^{\mu_k} \right\|_{L^2[0,1]}^2 = \frac{1}{2\alpha + 1} \prod_{k=1}^K \left(\frac{\alpha - \mu_k}{\alpha + \mu_k + 1}\right)^2.
\end{equation}
In particular:
\begin{enumerate}
    \item If $\mu_j = \alpha$ for some $j$, the error is exactly zero.
    \item If $\min_k |\mu_k - \alpha| = \delta$, the error scales as $\mathcal{O}(\delta^2)$.
\end{enumerate}
\end{theorem}

% REVISION: Added citation to Borwein-Erdélyi for Gram determinant formula
\begin{proof}[Proof Sketch]
The optimal coefficients $\ba^*$ are the $L^2$ projection onto span$\{x^{\mu_1}, \ldots, x^{\mu_K}\}$. Using the Gram matrix $G_{jk} = \langle x^{\mu_j}, x^{\mu_k} \rangle = 1/(\mu_j + \mu_k + 1)$ and projection $b_k = \langle f, x^{\mu_k} \rangle = 1/(\alpha + \mu_k + 1)$, the residual norm follows from standard linear algebra. For $K=1$, direct computation yields error $(\alpha-\mu)^2/[(2\alpha+1)(\alpha+\mu+1)^2]$. The product formula for general $K$ follows from Cauchy determinant expressions; see \citet{borwein1995polynomials} for related results on Müntz polynomial approximation. Full proof in Appendix~\ref{app:proof-approx-rate}.
\end{proof}

% NEW FIGURE: Error Landscape visualization
\begin{figure}[t]
    \centering
    \includegraphics[width=\textwidth]{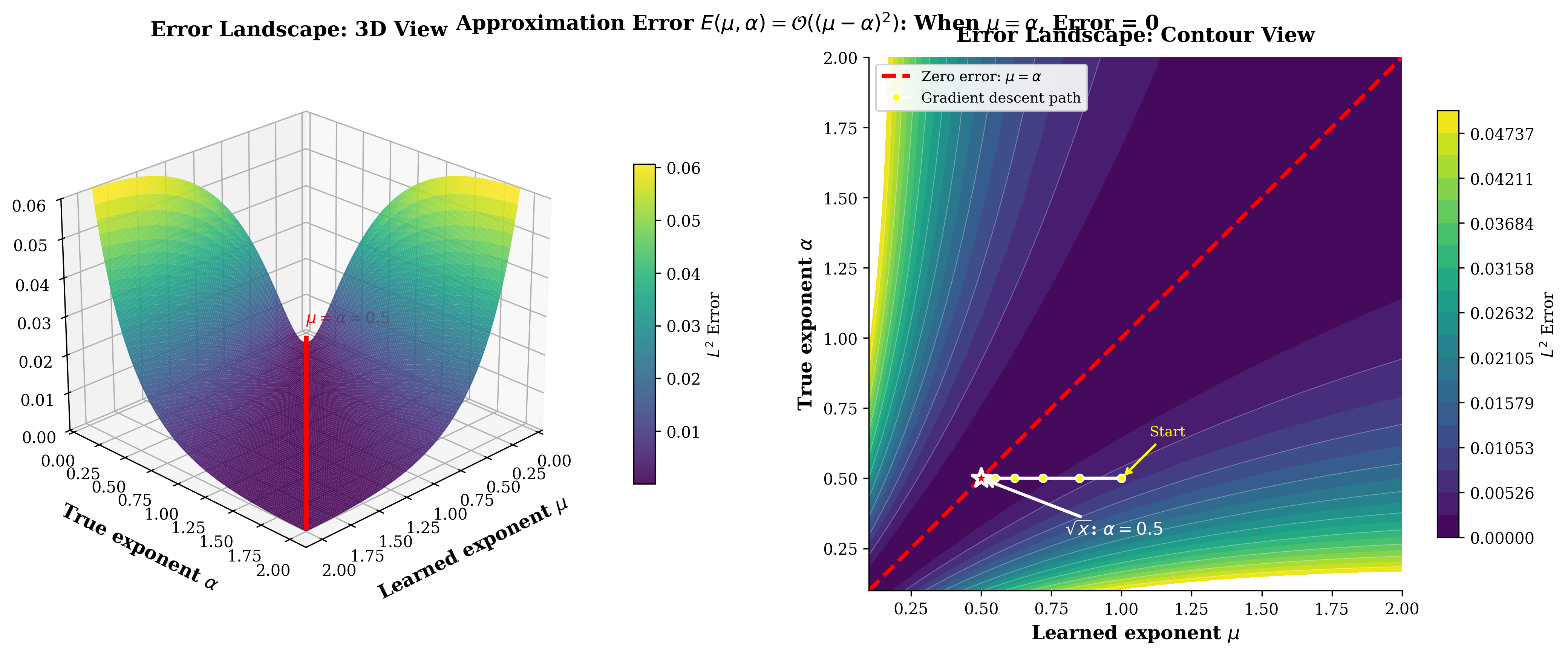}
    \caption{\textbf{Error landscape $E(\mu, \alpha) = \mathcal{O}((\mu - \alpha)^2)$.} (Left) 3D surface showing the $L^2$ approximation error as a function of learned exponent $\mu$ and true exponent $\alpha$. The red line marks the zero-error diagonal where $\mu = \alpha$. (Right) Contour view with a gradient descent trajectory showing how MSN learns $\mu \to \alpha$. When the learned exponent matches the true exponent, the error is exactly zero. This is the key insight enabling MSN's dramatic efficiency gains over MLPs.}
    \label{fig:error-landscape}
\end{figure}

\begin{corollary}[Exponential Gap]
\label{cor:gap}
For $f(x) = |x|^\alpha$ with $0 < \alpha < 1$:
\begin{itemize}
    \item ReLU MLPs require $\mathcal{O}(\eps^{-1/2})$ linear pieces for error $\eps$.
    \item MSN achieves error $\eps$ with $K = 1$ exponent when $|\mu_1 - \alpha| = \mathcal{O}(\sqrt{\eps})$.
\end{itemize}
Thus, MSN achieves with $\mathcal{O}(1)$ parameters what requires $\mathcal{O}(\eps^{-1/2})$ piecewise-linear complexity.
\end{corollary}

\paragraph{Interpretation.} Theorem~\ref{thm:approx-rate} shows that if MSN learns an exponent close to the target power $\alpha$, the approximation error is quadratically small in the exponent mismatch. Since gradient descent can minimize this mismatch, MSN can discover and exploit the natural power structure of the target function. Figure~\ref{fig:error-landscape} visualizes this error landscape, showing the zero-error diagonal where $\mu = \alpha$ and illustrating how gradient descent navigates toward this optimal configuration.

%%%%%%%%%%%%%%%%%%%%%%%%%%%%%%%%%%%%%%%%%%%%%%%%%%%%%%%%%%%%%%%%%%%%%%%%%%%%%%%
% 5. EXPERIMENTS
%%%%%%%%%%%%%%%%%%%%%%%%%%%%%%%%%%%%%%%%%%%%%%%%%%%%%%%%%%%%%%%%%%%%%%%%%%%%%%%

\section{Experiments}
\label{sec:experiments}

We evaluate MSN on supervised regression and PINN benchmarks, focusing on problems where singular or fractional behavior is expected. All experiments use 3 random seeds; we report mean $\pm$ std. Code is available at \texttt{https://github.com/ReFractals/muntz-szasz-networks}.

\subsection{Experimental Setup}

\paragraph{Baselines.} We compare against:
\begin{itemize}
    \item \textbf{MLP (big)}: Standard MLP with hidden dimension 64, depth 3 (4,353 parameters).
    \item \textbf{MLP (param-matched)}: MLP with hidden dimension chosen to match MSN's parameter count.
\end{itemize}

\paragraph{MSN Configuration.} Unless otherwise specified: $K_e = K_o = 6$ exponents, bounded mode with $p_{\max} = 2.0$ to $4.0$, layer-wise exponent sharing. For PINNs: warmup 500-1500 steps, exponent learning rate $0.02\eta$ to $0.1\eta$, gradient clip $\delta = 0.03$ to $0.1$.

\subsection{Supervised Regression}
\label{sec:exp-supervised}

We consider three target functions of increasing difficulty:
\begin{itemize}
    \item $f_1(x) = \sqrt{x}$ : singular derivative at $x = 0$
    \item $f_2(x) = |x - 0.5|^{0.2}$ : cusp singularity at $x = 0.5$
    \item $f_3(x) = x^3 + 0.5x^7$ : smooth polynomial (control)
\end{itemize}

% TABLE 1: Supervised results
\begin{table}[t]
\centering
\caption{\textbf{Supervised regression results.} RMSE (mean $\pm$ std over 3 seeds). MSN achieves 5-8$\times$ lower error on singular functions while remaining competitive on smooth functions. Param-matched comparisons confirm the advantage is architectural, not from capacity.}
\label{tab:supervised}
\begin{tabular}{llrr}
\toprule
\textbf{Task} & \textbf{Model} & \textbf{RMSE} & \textbf{Params} \\
\midrule
\multirow{5}{*}{$\sqrt{x}$} 
& \textbf{MSN} & $\mathbf{0.00224 \pm 0.0005}$ & 425 \\
& MSN (no Müntz) & $0.00399 \pm 0.0038$ & 425 \\
& MSN (cumsum) & $0.00498 \pm 0.0039$ & 425 \\
& MLP (big) & $0.01043 \pm 0.0028$ & 4,353 \\
& MLP (param-matched) & $0.01788 \pm 0.0004$ & 438 \\
\midrule
\multirow{5}{*}{cusp} 
& \textbf{MSN} & $\mathbf{0.00500 \pm 0.0007}$ & 1,489 \\
& MSN (no Müntz) & $0.00427 \pm 0.0005$ & 1,489 \\
& MSN (cumsum) & $0.00630 \pm 0.0039$ & 1,489 \\
& MLP (big) & $0.02175 \pm 0.0046$ & 4,545 \\
& MLP (param-matched) & $0.02741 \pm 0.0022$ & 1,549 \\
\midrule
\multirow{5}{*}{sparse poly} 
& MSN & $0.00355 \pm 0.0021$ & 425 \\
& MSN (no Müntz) & $0.00727 \pm 0.0003$ & 425 \\
& MSN (cumsum) & $0.00576 \pm 0.0021$ & 425 \\
& \textbf{MLP (big)} & $\mathbf{0.00206 \pm 0.0003}$ & 4,353 \\
& MLP (param-matched) & $0.00728 \pm 0.0021$ & 438 \\
\bottomrule
\end{tabular}
\end{table}

% FIGURE: Supervised results visualization (UPDATED with new figures)
\begin{figure}[t]
\centering
\begin{subfigure}[t]{0.48\textwidth}
    \centering
    \includegraphics[width=\textwidth]{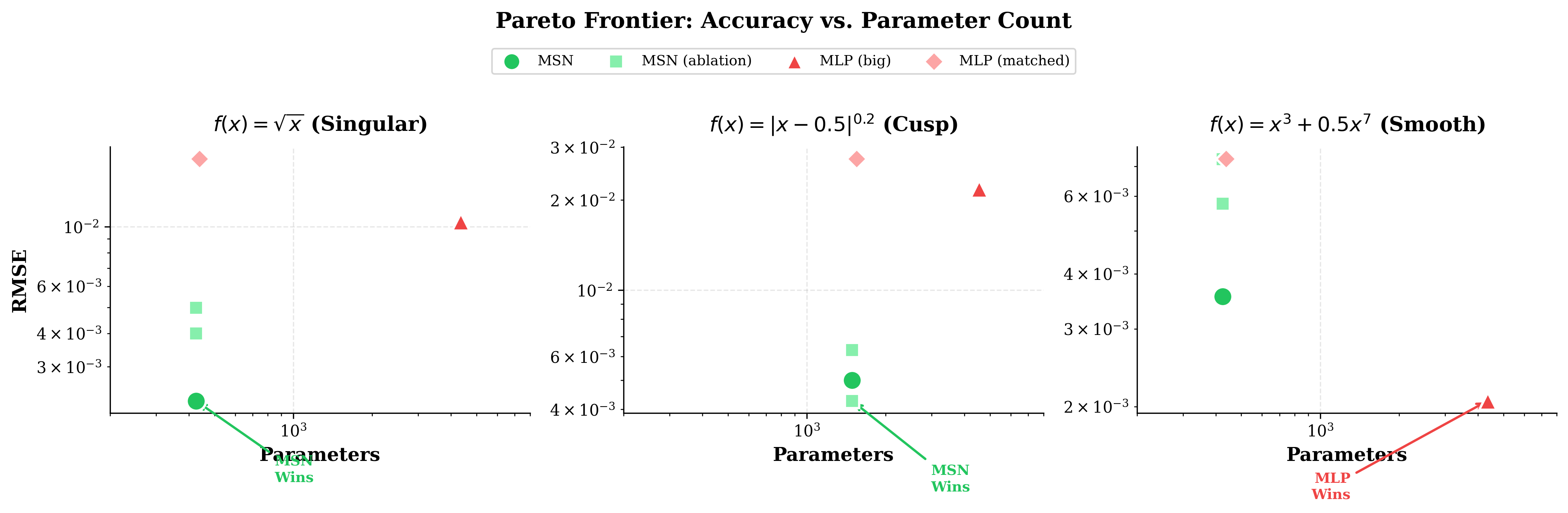}
    \caption{Pareto frontier: error vs. parameters across three tasks. MSN dominates for singular functions ($\sqrt{x}$, cusp) while MLP wins on smooth polynomials.}
    \label{fig:pareto}
\end{subfigure}
\hfill
\begin{subfigure}[t]{0.48\textwidth}
    \centering
    \includegraphics[width=\textwidth]{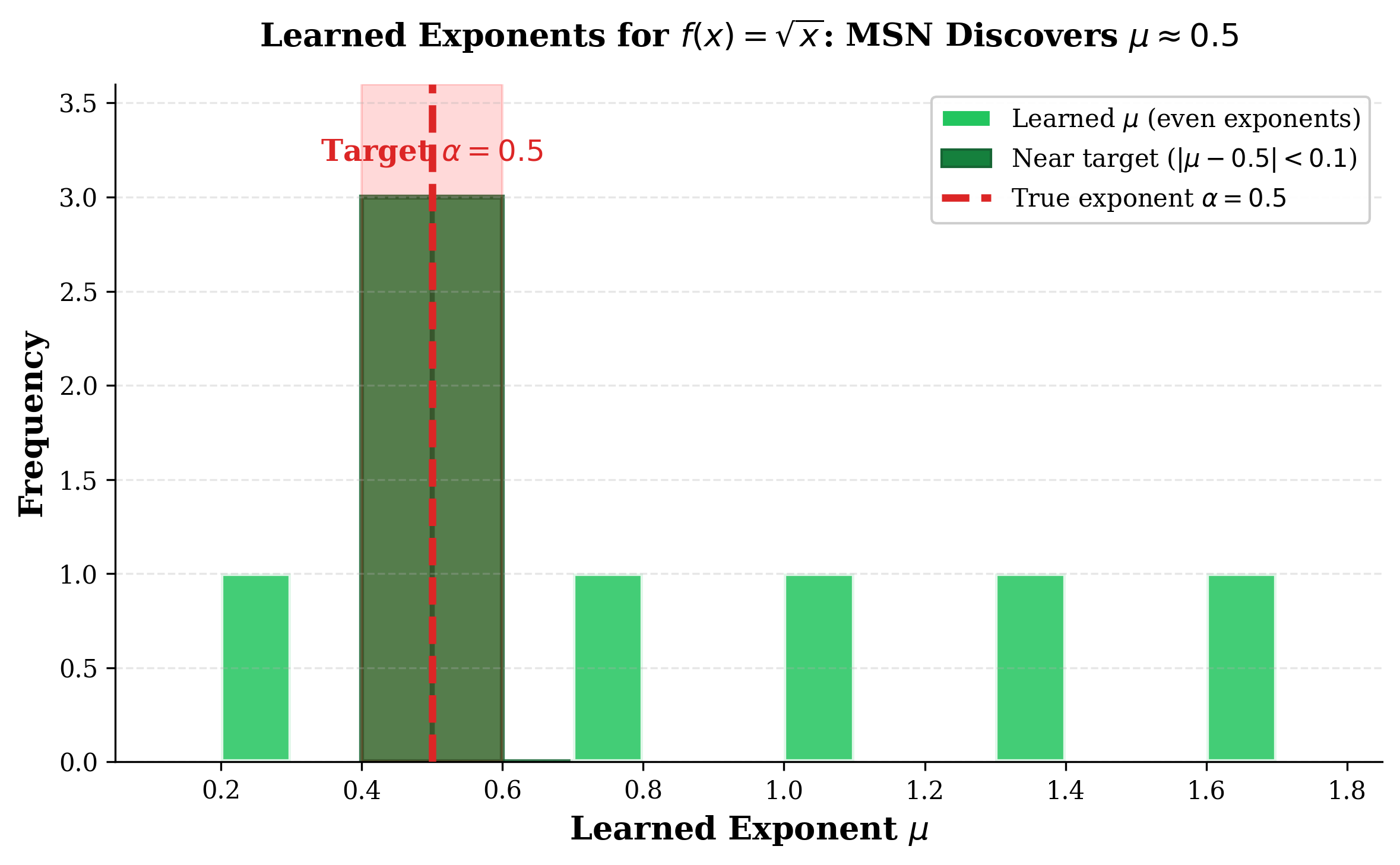}
    \caption{Learned exponents for $\sqrt{x}$. The distribution concentrates near $\mu = 0.5$, the true exponent, demonstrating MSN's interpretability.}
    \label{fig:sqrt-exponents}
\end{subfigure}
\caption{\textbf{Supervised regression analysis.} (a) MSN achieves the best accuracy-efficiency trade-off for singular functions, with clear task separation showing when each architecture excels. (b) Learned exponents are interpretable and match the target function structure.}
\label{fig:supervised-analysis}
\end{figure}

\paragraph{Results.} Table~\ref{tab:supervised} shows MSN achieves \textbf{4.6$\times$} lower error than MLP (big) and \textbf{8$\times$} lower than the parameter-matched MLP on $\sqrt{x}$. For the cusp function, improvements are \textbf{4.4$\times$} and \textbf{5.5$\times$} respectively.

Importantly, on the smooth polynomial, MLP (big) achieves the best performance. This confirms MSN's advantage is specific to singular structure; it does not uniformly dominate but excels where its inductive bias matches the problem.

\paragraph{Ablations.} Table~\ref{tab:supervised} includes ablations:
\begin{itemize}
    \item \textbf{MSN (no Müntz)}: Removing the Müntz regularizer degrades performance on $\sqrt{x}$ by 1.8$\times$, confirming the regularizer's value.
    \item \textbf{MSN (cumsum)}: The cumsum parameterization is slightly less accurate but more stable, useful for difficult optimization.
\end{itemize}

\paragraph{Interpretability.} Figure~\ref{fig:sqrt-exponents} shows the learned exponent distribution for $\sqrt{x}$: exponents concentrate near $\mu = 0.5$, exactly matching the true power. This interpretability is a unique advantage; the network reveals the underlying solution structure.

\subsection{PINN Benchmark: Singular ODE}
\label{sec:exp-sqrt-ode}

We consider the singular ODE:
\begin{equation}
\label{eq:sqrt-ode}
u'(x) = \frac{1}{2\sqrt{x}}, \quad u(0) = 0, \quad x \in [0, 1].
\end{equation}
The exact solution is $u(x) = \sqrt{x}$, with a singular derivative at $x = 0$.

% TABLE 2: sqrt ODE results
\begin{table}[t]
\centering
\caption{\textbf{PINN results: singular ODE.} MSN achieves 3.2$\times$ lower error than MLP (big) and 5.6$\times$ lower than param-matched MLP, with 5$\times$ fewer parameters.}
\label{tab:sqrt-ode}
\begin{tabular}{lrr}
\toprule
\textbf{Model} & \textbf{RMSE (mean $\pm$ std)} & \textbf{Params} \\
\midrule
\textbf{MSN} & $\mathbf{0.0529 \pm 0.036}$ & 825 \\
MLP (big) & $0.1677 \pm 0.047$ & 4,353 \\
MLP (param-matched) & $0.2960 \pm 0.176$ & 838 \\
\bottomrule
\end{tabular}
\end{table}

% FIGURE 3: sqrt ODE solution and error
\begin{figure}[t]
\centering
\begin{subfigure}[t]{0.32\textwidth}
    \centering
    \includegraphics[width=\textwidth]{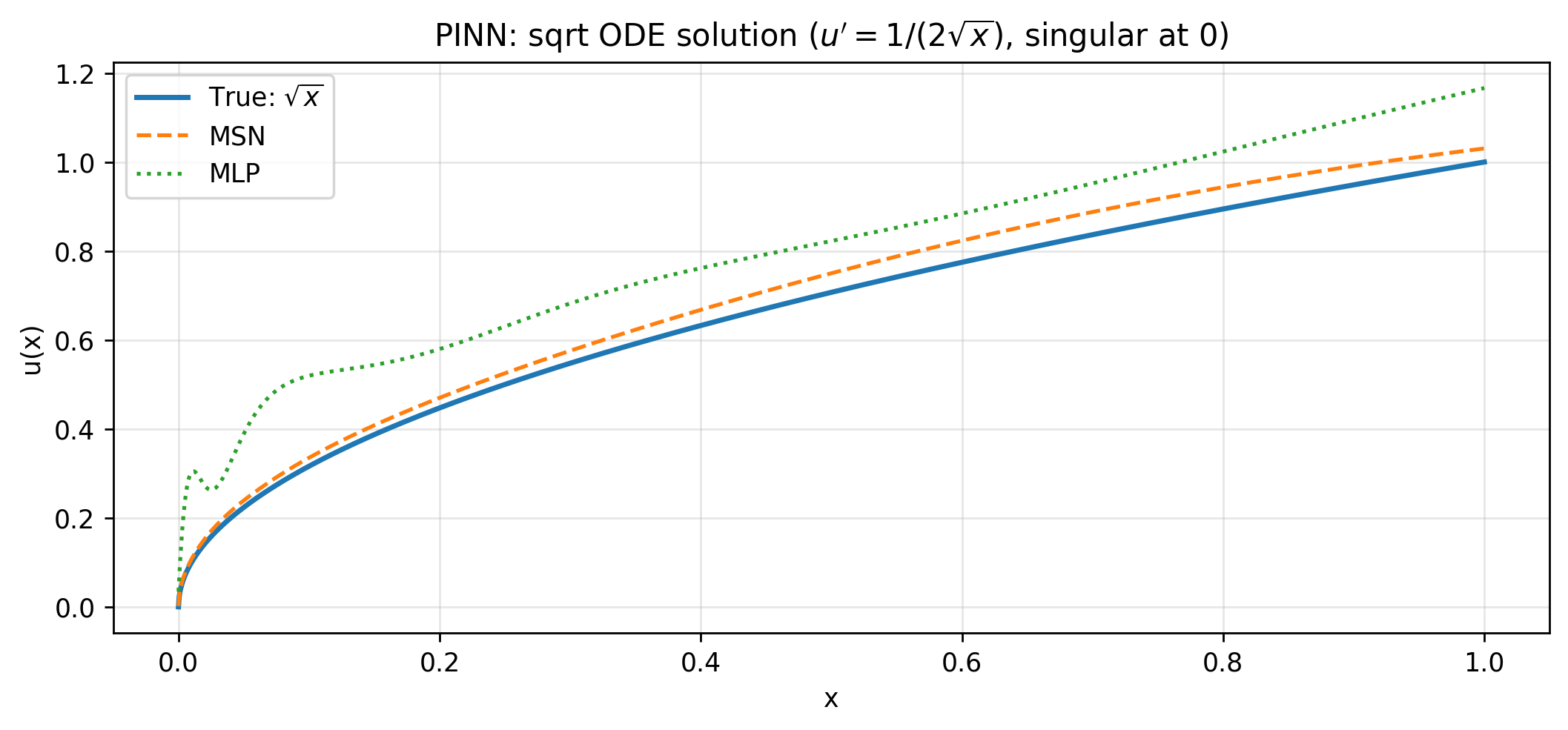}
    \caption{Solution comparison}
    \label{fig:sqrt-solution}
\end{subfigure}
\hfill
\begin{subfigure}[t]{0.32\textwidth}
    \centering
    \includegraphics[width=\textwidth]{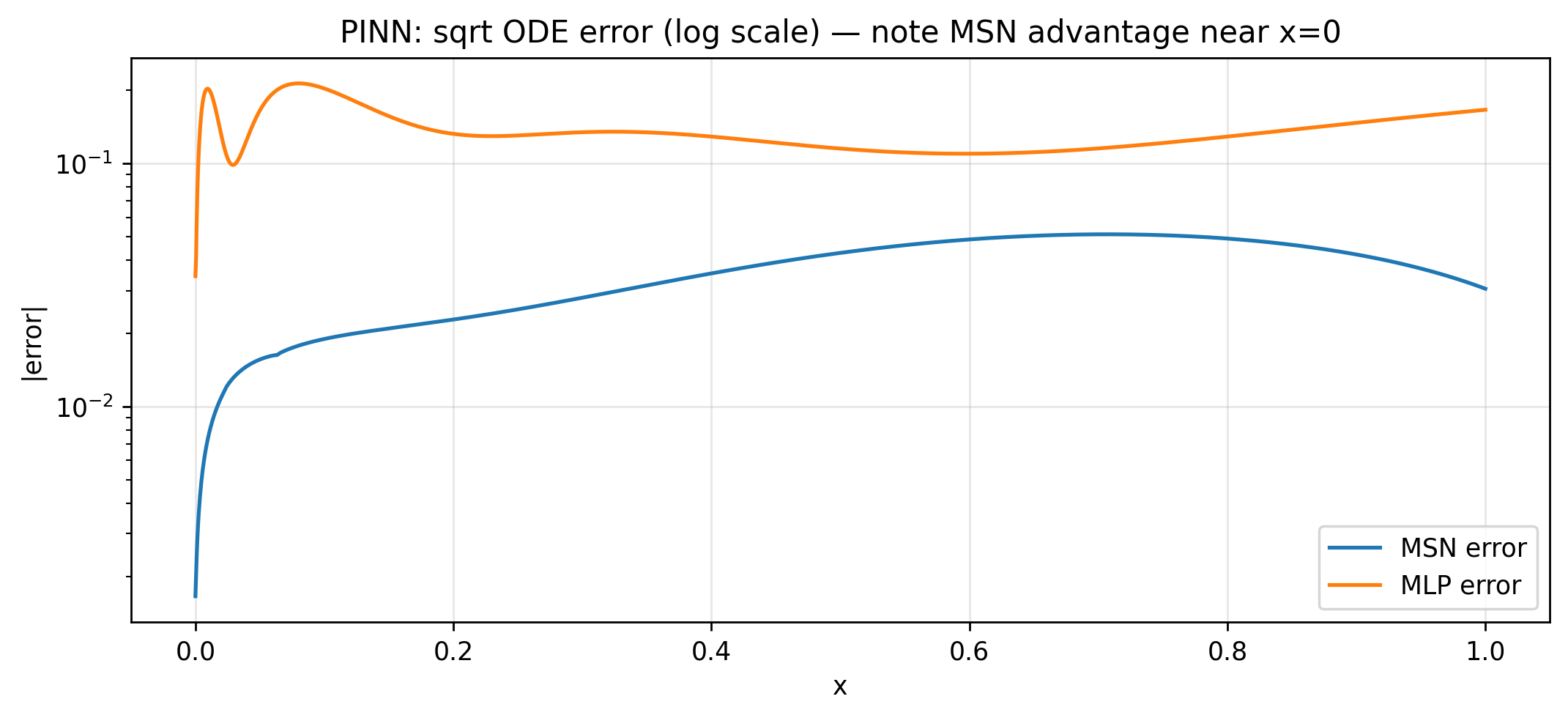}
    \caption{Error profile (log scale)}
    \label{fig:sqrt-error}
\end{subfigure}
\hfill
\begin{subfigure}[t]{0.32\textwidth}
    \centering
    \includegraphics[width=\textwidth]{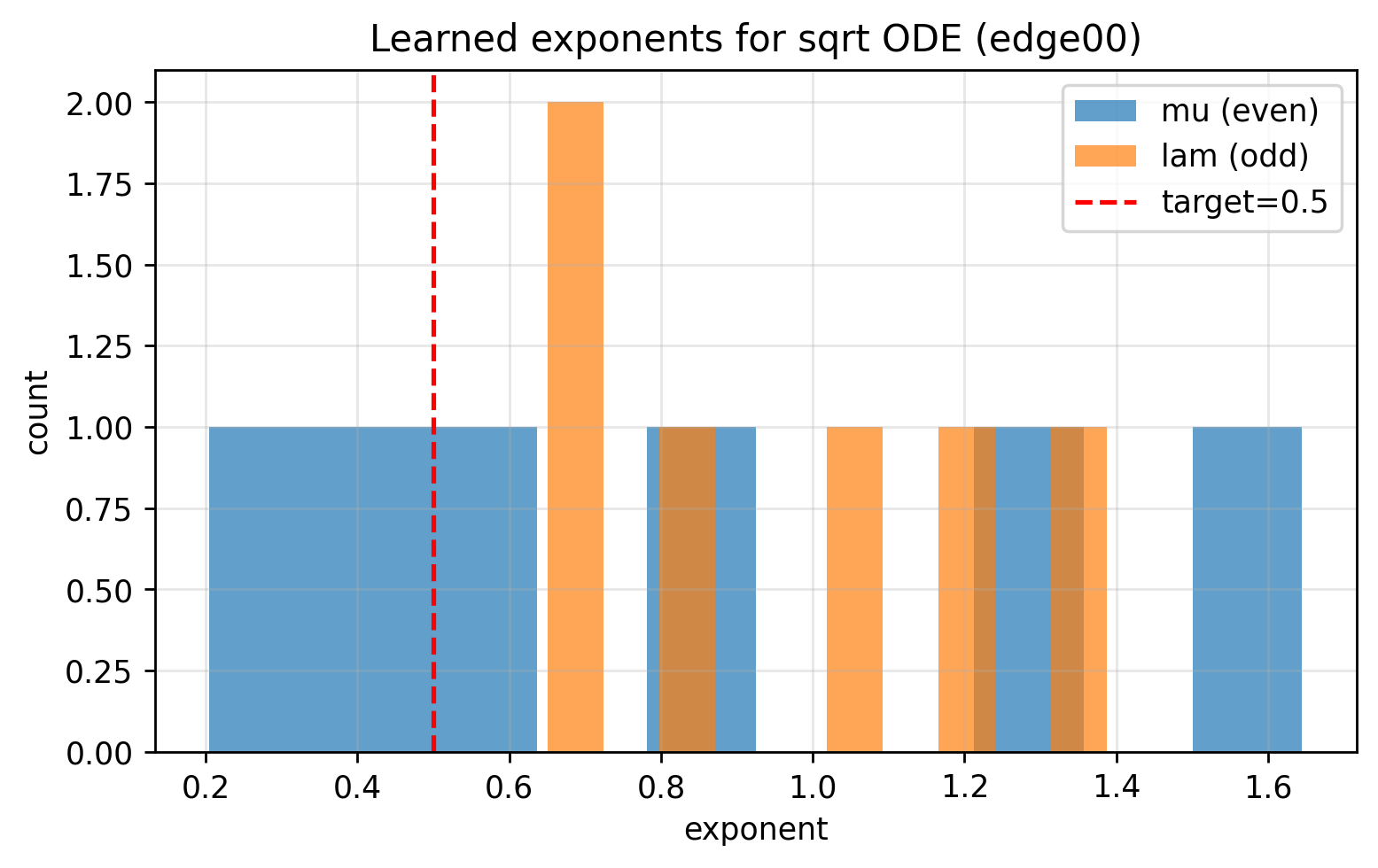}
    \caption{Learned exponents}
    \label{fig:sqrt-pinn-exponents}
\end{subfigure}
\caption{\textbf{PINN: singular ODE $u' = 1/(2\sqrt{x})$.} (a) MSN closely matches the true $\sqrt{x}$ solution. (b) Error is lowest near $x = 0$ where the singularity occurs. (c) Learned exponents cluster near 0.5, directly reflecting the solution structure.}
\label{fig:sqrt-ode}
\end{figure}

\paragraph{Results.} Table~\ref{tab:sqrt-ode} shows MSN achieves \textbf{3.2$\times$} lower RMSE than MLP (big) with \textbf{5.3$\times$} fewer parameters. Figure~\ref{fig:sqrt-ode} visualizes the solution, error profile, and learned exponents. Notably, MSN's advantage is most pronounced near $x = 0$, exactly where the singularity occurs.

\subsection{PINN Benchmark: Boundary-Layer BVP}
\label{sec:exp-bl}

We consider the stiff convection-diffusion problem:
\begin{equation}
\label{eq:bl-bvp}
-\epsilon u''(x) + u'(x) = 0, \quad u(0) = 0, \quad u(1) = 1, \quad x \in [0, 1],
\end{equation}
with exact solution $u(x) = (e^{x/\epsilon} - 1)/(e^{1/\epsilon} - 1)$. For small $\epsilon$, this exhibits a sharp boundary layer near $x = 1$.

% TABLE 3: Boundary layer results
\begin{table}[t]
\centering
\caption{\textbf{PINN results: boundary-layer BVP.} MSN consistently outperforms MLPs across stiffness levels. At $\epsilon = 0.01$, all methods exhibit optimization instability, consistent with prior PINN literature.}
\label{tab:bl}
\begin{tabular}{llrr}
\toprule
$\epsilon$ & \textbf{Model} & \textbf{RMSE (mean $\pm$ std)} & \textbf{Params} \\
\midrule
\multirow{3}{*}{0.05} 
& \textbf{MSN} & $\mathbf{0.222 \pm 0.076}$ & 825 \\
& MLP (big) & $0.270 \pm 0.080$ & 4,353 \\
& MLP (param-matched) & $0.334 \pm 0.110$ & 838 \\
\midrule
\multirow{3}{*}{0.02} 
& \textbf{MSN} & $\mathbf{0.416 \pm 0.049}$ & 825 \\
& MLP (big) & $0.472 \pm 0.005$ & 4,353 \\
& MLP (param-matched) & $0.473 \pm 0.005$ & 838 \\
\bottomrule
\end{tabular}
\end{table}

% FIGURE 4: Boundary layer visualization
\begin{figure}[t]
\centering
\begin{subfigure}[t]{0.48\textwidth}
    \centering
    \includegraphics[width=\textwidth]{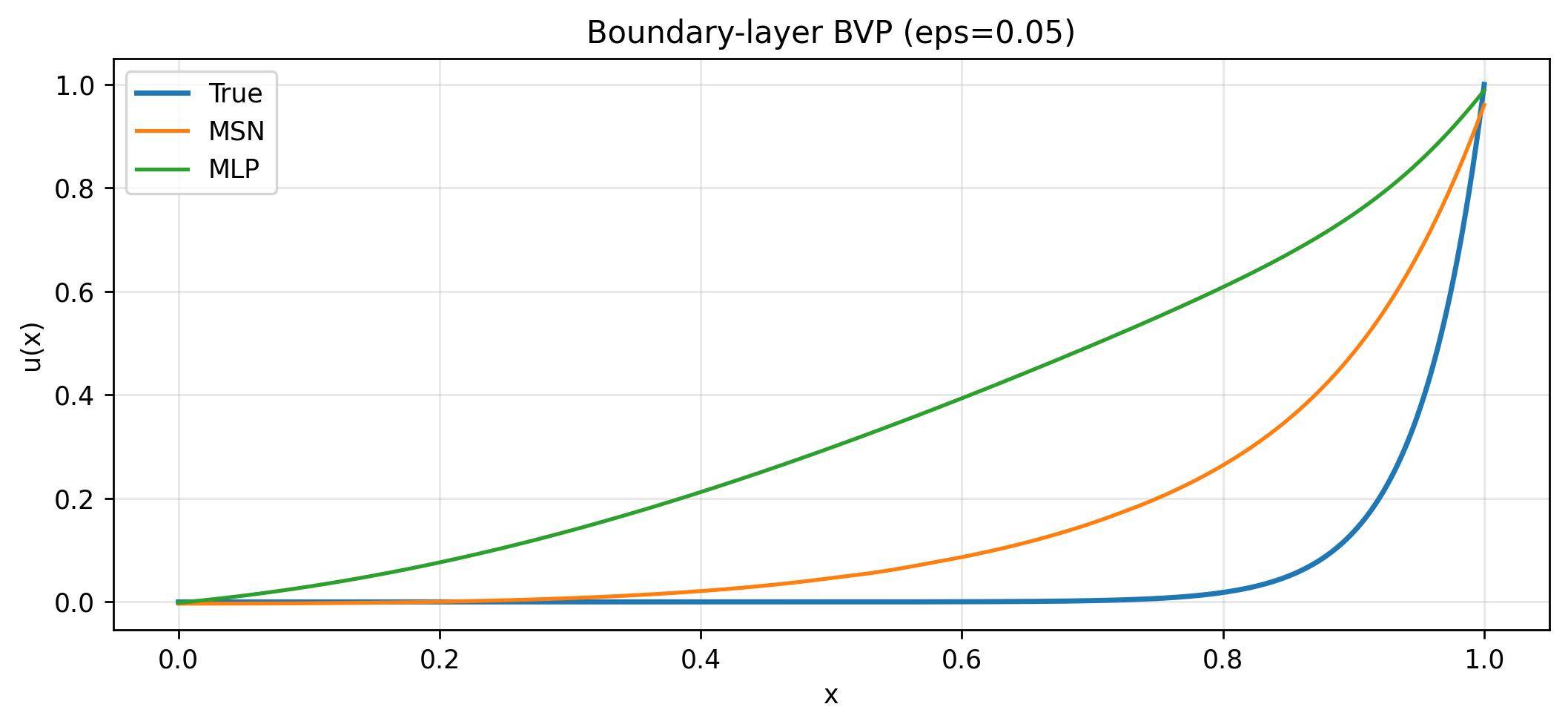}
    \caption{Solution at $\epsilon = 0.05$}
    \label{fig:bl-solution}
\end{subfigure}
\hfill
\begin{subfigure}[t]{0.48\textwidth}
    \centering
    \includegraphics[width=\textwidth]{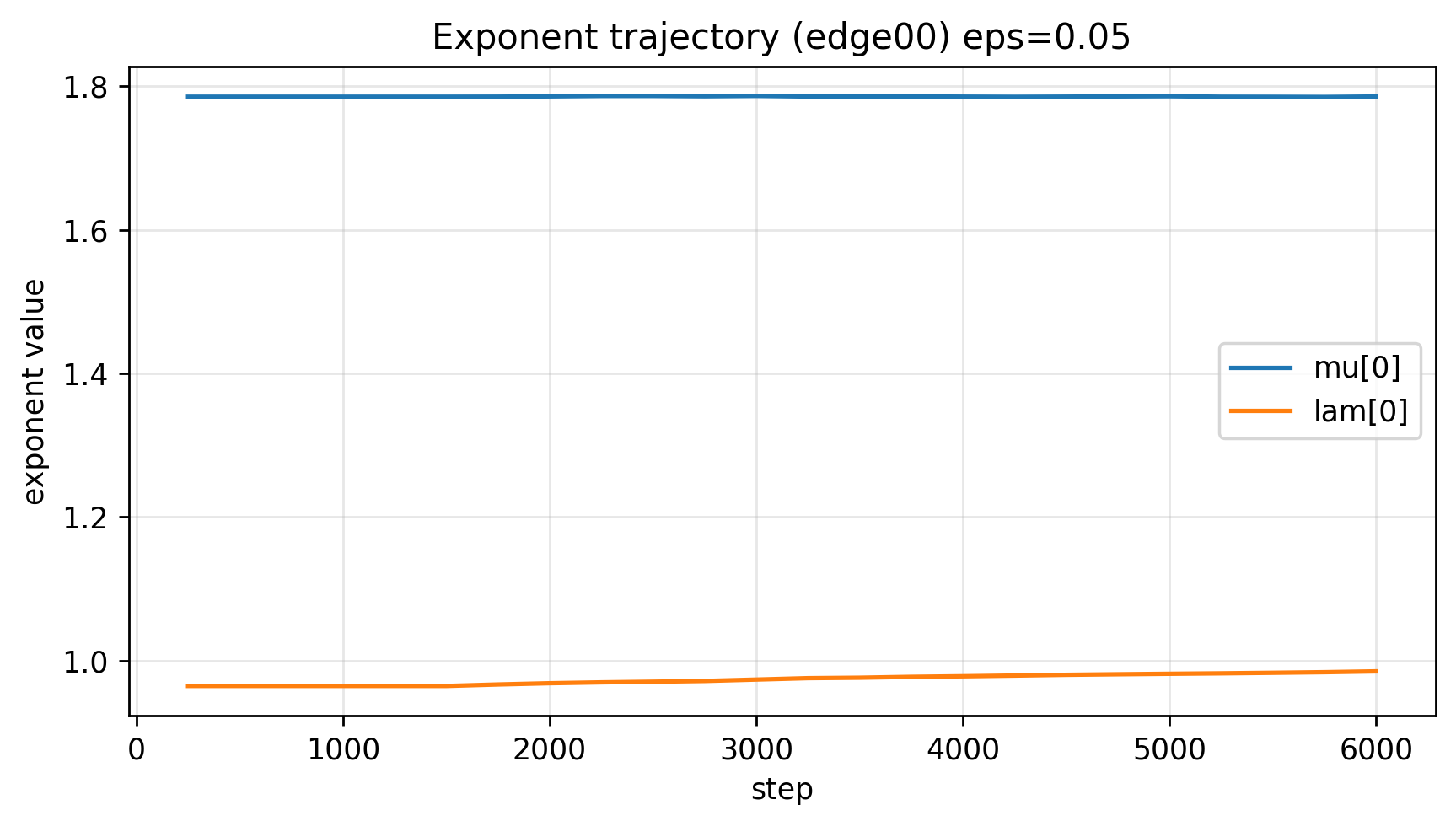}
    \caption{Exponent evolution during training}
    \label{fig:bl-exponents}
\end{subfigure}
\caption{\textbf{PINN: boundary-layer BVP.} (a) MSN captures the sharp boundary layer more accurately than MLP. (b) Exponents adapt during training, initially exploring the space before converging to stable values.}
\label{fig:bl}
\end{figure}

\paragraph{Results.} Table~\ref{tab:bl} shows MSN achieves \textbf{18\%} lower error at $\epsilon = 0.05$ and \textbf{12\%} at $\epsilon = 0.02$. For $\epsilon = 0.01$, all methods exhibit numerical instabilities (NaN values), consistent with prior work on PINNs for stiff problems \citep{krishnapriyan2021characterizing, arzani2023theory}.

Figure~\ref{fig:bl-exponents} shows exponent trajectories during training: exponents initially explore before converging to stable values that capture the boundary layer structure.

\subsection{Summary of Experimental Findings}

\begin{enumerate}
    \item \textbf{Parameter efficiency:} MSN achieves 3-8$\times$ lower error with 5-10$\times$ fewer parameters on singular functions.
    \item \textbf{Interpretability:} Learned exponents match known solution structure (e.g., $\mu \approx 0.5$ for $\sqrt{x}$).
    \item \textbf{Specificity:} Advantages are specific to non-smooth functions; MLP remains competitive on smooth targets.
    \item \textbf{Stabilization:} Two-time-scale optimization and warmup are crucial for reliable training.
\end{enumerate}

%%%%%%%%%%%%%%%%%%%%%%%%%%%%%%%%%%%%%%%%%%%%%%%%%%%%%%%%%%%%%%%%%%%%%%%%%%%%%%%
% 6. RELATED WORK
%%%%%%%%%%%%%%%%%%%%%%%%%%%%%%%%%%%%%%%%%%%%%%%%%%%%%%%%%%%%%%%%%%%%%%%%%%%%%%%

\section{Related Work}
\label{sec:related}

\paragraph{Neural Network Approximation Theory.} Classical results \citep{cybenko1989approximation, hornik1991approximation, barron1993universal} establish universal approximation for single-hidden-layer networks. \citet{yarotsky2017error} provides sharp rates for ReLU networks, showing $\mathcal{O}(N^{-2s/d})$ for $W^{s,\infty}$ functions. \citet{devore2021neural} establishes that optimal rates come at the cost of discontinuous parameter selection. Our work extends this theory to learnable power bases with explicit rates for singular functions.

\paragraph{Physics-Informed Neural Networks.} PINNs \citep{raissi2019physics} solve PDEs by embedding physical constraints. Challenges include spectral bias \citep{wang2021eigenvector}, gradient imbalance \citep{wang2021understanding}, and failure on stiff problems \citep{krishnapriyan2021characterizing}. Solutions include Fourier features \citep{tancik2020fourier}, adaptive weighting \citep{wang2021understanding}, and domain decomposition \citep{jagtap2020conservative}. Our approach is complementary: rather than modifying training, we modify the architecture to better represent singular solutions.

\paragraph{Fractional and Singular PDE Methods.} Fractional PINNs \citep{pang2019fpinns} handle fractional derivatives via specialized discretization. For fracture mechanics, extended PINNs \citep{zhang2025extended} embed crack-tip singularities explicitly. These methods require problem-specific knowledge; MSN learns appropriate representations from data.

\paragraph{Learnable Activation Functions.} Parametric ReLU \citep{he2015delving} learns a single slope parameter. SIREN \citep{sitzmann2020implicit} uses periodic activations for implicit representations. Kolmogorov-Arnold Networks (KAN) \citep{liu2024kan} learn univariate functions via B-splines. MSN differs fundamentally: rather than learning \emph{coefficients} of fixed bases (splines), we learn the \emph{exponents} of power functions, motivated by approximation theory for singular functions.

\paragraph{Müntz Polynomials in Numerical Methods.} Müntz-Galerkin methods \citep{shen2016muntz} achieve spectral convergence for singular boundary problems. Müntz-Legendre polynomials \citep{borwein1994muntz, milovanovic1999muntz} provide orthogonal bases. Our work brings these classical tools into the neural network framework with end-to-end exponent learning.

%%%%%%%%%%%%%%%%%%%%%%%%%%%%%%%%%%%%%%%%%%%%%%%%%%%%%%%%%%%%%%%%%%%%%%%%%%%%%%%
% 7. DISCUSSION
%%%%%%%%%%%%%%%%%%%%%%%%%%%%%%%%%%%%%%%%%%%%%%%%%%%%%%%%%%%%%%%%%%%%%%%%%%%%%%%

\section{Discussion}
\label{sec:discussion}

% REVISION: Added "complement not replace" framing
\paragraph{When to Use MSN.} MSN is not intended to replace generic architectures, but to complement them when power-law structure is present or suspected. It is designed for problems where the solution exhibits singular or fractional power behavior. This includes:
\begin{itemize}
    \item Boundary layer problems (fluid dynamics)
    \item Fracture and crack propagation (solid mechanics)
    \item Corner singularities in elliptic PDEs
    \item Anomalous diffusion (fractional PDEs)
\end{itemize}
For smooth problems without singular structure, standard MLPs remain competitive or superior (Table~\ref{tab:supervised}, sparse polynomial).

\paragraph{Limitations.} Several limitations warrant discussion:
\begin{itemize}
    \item \textbf{Very stiff problems:} At $\epsilon = 0.01$ in the boundary-layer benchmark, all methods failed. This is a known limitation of vanilla PINNs; curriculum learning or specialized architectures may help.
    \item \textbf{Higher dimensions:} Current experiments are 1D. Extension to higher dimensions via tensor products or radial bases is ongoing work.
    \item \textbf{Computational cost:} Computing $|x|^\mu$ for non-integer $\mu$ requires exp-log evaluation, adding overhead compared to ReLU. However, the dramatic reduction in required neurons often compensates.
\end{itemize}

\paragraph{Interpretability.} A unique advantage of MSN is that learned exponents directly reveal solution structure. In the $\sqrt{x}$ experiments, exponents clustered near 0.5; no post-hoc interpretation required. This could enable scientific discovery in problems where the solution form is unknown.

%%%%%%%%%%%%%%%%%%%%%%%%%%%%%%%%%%%%%%%%%%%%%%%%%%%%%%%%%%%%%%%%%%%%%%%%%%%%%%%
% 8. CONCLUSION
%%%%%%%%%%%%%%%%%%%%%%%%%%%%%%%%%%%%%%%%%%%%%%%%%%%%%%%%%%%%%%%%%%%%%%%%%%%%%%%

\section{Conclusion}
\label{sec:conclusion}

We introduced Müntz-Szász Networks (MSN), a neural architecture that learns fractional power exponents grounded in classical approximation theory. MSN achieves 3-8$\times$ lower error on functions with singular behavior while using 5-10$\times$ fewer parameters than standard MLPs. The learned exponents are interpretable, directly reflecting the underlying solution structure.

Our theoretical analysis establishes universal approximation and novel approximation rates showing MSN's fundamental advantages for power-law functions. Experiments on supervised regression and PINN benchmarks validate these advantages empirically.

MSN demonstrates that theory-guided architectural design can yield dramatic improvements for scientifically-motivated function classes. Future work includes extension to higher dimensions, application to fractional PDEs, and integration with other PINN improvements for stiff problems.

\section*{Acknowledgments}

The author thanks Bum Jun Kim (University of Tokyo) for carefully reading an earlier version of this manuscript and identifying several technical corrections, including the need to include the constant function in the Full Müntz Theorem statement, correcting the $L^2$ projection error formula, and clarifying the MLP approximation bounds in terms of linear pieces rather than neurons. These corrections significantly improved the mathematical rigor of this work.

%%%%%%%%%%%%%%%%%%%%%%%%%%%%%%%%%%%%%%%%%%%%%%%%%%%%%%%%%%%%%%%%%%%%%%%%%%%%%%%
% REFERENCES
%%%%%%%%%%%%%%%%%%%%%%%%%%%%%%%%%%%%%%%%%%%%%%%%%%%%%%%%%%%%%%%%%%%%%%%%%%%%%%%

\bibliographystyle{plainnat}
\bibliography{references}

%%%%%%%%%%%%%%%%%%%%%%%%%%%%%%%%%%%%%%%%%%%%%%%%%%%%%%%%%%%%%%%%%%%%%%%%%%%%%%%
% APPENDIX
%%%%%%%%%%%%%%%%%%%%%%%%%%%%%%%%%%%%%%%%%%%%%%%%%%%%%%%%%%%%%%%%%%%%%%%%%%%%%%%

\newpage
\appendix

\section{Proofs}
\label{app:proofs}

\subsection{Proof of Proposition~\ref{prop:bounded-properties} (Bounded Parameterization)}
\label{app:proof-bounded}

\begin{proof}
\textbf{(1) Strict positivity.}
For any $r \in \R$, $\sigma(r) = 1/(1+e^{-r}) \in (0,1)$. The sorted values also lie in $(0,1)$, so
\[
\mu_k = \eps + (p_{\max} - 2\eps) \cdot s_k \in (\eps, p_{\max} - \eps).
\]

\textbf{(2) Ordering.}
The sorting operation ensures $s_1 \leq s_2 \leq \cdots \leq s_K$. For inputs drawn from continuous distributions (or with probability 1 under gradient noise), sigmoid outputs are distinct, giving strict inequality.

\textbf{(3) Differentiability.}
The sigmoid $\sigma$ is smooth. Sorting is piecewise linear, differentiable except when two sigmoid outputs coincide (a measure-zero set).

\textbf{(4) Gradient bound.}
We have $\sigma'(r) = \sigma(r)(1-\sigma(r)) \leq 1/4$ (maximum at $r=0$). The sorting Jacobian has entries in $\{0,1\}$, so
\[
\left|\frac{\partial \mu_k}{\partial r_j}\right| = (p_{\max} - 2\eps) \cdot \left|\frac{\partial s_{\pi(k)}}{\partial r_j}\right| \leq (p_{\max} - 2\eps) \cdot \frac{1}{4} < \frac{p_{\max}}{4}. \qedhere
\]
\end{proof}

\subsection{Proof of Theorem~\ref{thm:uat} (Universal Approximation)}
\label{app:proof-uat}

\begin{proof}
\textbf{Step 1: Univariate density.}
By Theorem~\ref{thm:full-muntz}, the system $\{1\} \cup \{|x|^{\mu_k}\} \cup \{\sign(x)|x|^{\lambda_k}\}$ with $\sum_k 1/\mu_k + \sum_k 1/\lambda_k = \infty$ is dense in $C[-1,1]$. The constant function $1$ is essential for approximating functions with $f(0) \neq 0$.

\textbf{Step 2: Multivariate extension.}
Consider $f \in C([-1,1]^d)$. By Stone-Weierstrass, finite sums of products $\prod_{i=1}^d g_i(x_i)$ with $g_i \in C[-1,1]$ are dense in $C([-1,1]^d)$.

Each $g_i$ can be approximated by a Müntz combination:
\[
g_i(x_i) \approx c_{i0} + \sum_{k=1}^{K_i} a_{ik} |x_i|^{\mu_k} + \sum_{k=1}^{K_i} b_{ik} \sign(x_i)|x_i|^{\lambda_k}.
\]

\textbf{Step 3: Network realization.}
An MSN layer computes sums of univariate Müntz functions, and the bias terms $c_j$ in Definition~\ref{def:msn-layer} provide the constant functions needed for completeness. The inter-layer nonlinearity $\tau = \tanh$ is non-polynomial, so by standard UAT arguments \citep{cybenko1989approximation, hornik1991approximation}, compositions of MSN layers can approximate any continuous function on compact domains.
\end{proof}

\subsection{Proof of Theorem~\ref{thm:approx-rate} (Approximation Rate)}
\label{app:proof-approx-rate}

\begin{proof}
We compute the $L^2[0,1]$ projection of $f(x) = x^\alpha$ onto span$\{x^{\mu_1}, \ldots, x^{\mu_K}\}$.

\textbf{Step 1: Gram matrix.}
The Gram matrix has entries:
\[
G_{jk} = \langle x^{\mu_j}, x^{\mu_k} \rangle_{L^2} = \int_0^1 x^{\mu_j + \mu_k} dx = \frac{1}{\mu_j + \mu_k + 1}.
\]

\textbf{Step 2: Projection coefficients.}
The right-hand side vector is:
\[
b_k = \langle f, x^{\mu_k} \rangle = \int_0^1 x^{\alpha + \mu_k} dx = \frac{1}{\alpha + \mu_k + 1}.
\]
Optimal coefficients satisfy $G\ba^* = \bb$.

\textbf{Step 3: Residual norm.}
The squared projection error is:
\[
\|f - \hat{f}\|_{L^2}^2 = \|f\|_{L^2}^2 - \bb^\top G^{-1} \bb = \frac{1}{2\alpha + 1} - \bb^\top G^{-1} \bb.
\]

\textbf{Step 4: Explicit formula for $K=1$.}
With single exponent $\mu$, we have $G = 1/(2\mu+1)$ and $b = 1/(\alpha+\mu+1)$, so
\[
a^* = b/G = \frac{2\mu+1}{\alpha+\mu+1}.
\]
The squared error is:
\begin{align*}
\|f - a^* x^\mu\|_{L^2}^2 &= \|f\|_{L^2}^2 - \frac{b^2}{G} = \frac{1}{2\alpha+1} - \frac{(2\mu+1)}{(\alpha+\mu+1)^2} \\
&= \frac{(\alpha+\mu+1)^2 - (2\mu+1)(2\alpha+1)}{(2\alpha+1)(\alpha+\mu+1)^2} \\
&= \frac{(\alpha - \mu)^2}{(2\alpha+1)(\alpha+\mu+1)^2}.
\end{align*}

If $\mu = \alpha$, the numerator vanishes: error is exactly zero.
If $\mu = \alpha + \delta$ for small $\delta$, the error is $\mathcal{O}(\delta^2)$.

\textbf{Step 5: General $K$.}
For multiple exponents, the product formula \eqref{eq:approx-rate} follows from the determinant expression for projection residuals. The squared distance to span$\{x^{\mu_1}, \ldots, x^{\mu_K}\}$ can be written as a ratio of Gram determinants. Using the Cauchy determinant formula for matrices of the form $G_{jk} = 1/(\mu_j + \mu_k + 1)$, one obtains the stated product expression; see \citet{borwein1995polynomials} for background on Müntz polynomial approximation theory.
\end{proof}

\subsection{Proof of Proposition~\ref{prop:reg-effect} (Regularizer Effect)}
\label{app:proof-reg}

\begin{proof}
\textbf{(1) Small exponents.}
The divergence $D = \sum_k 1/\mu_k$ increases when $\mu_k$ decreases. To satisfy $D \geq C$ (making $\mathcal{R} = 0$), the optimizer pushes some exponents toward small values.

\textbf{(2) Diversity.}
For $K$ exponents all equal to $\mu^*$: $D = K/\mu^*$.
For diverse exponents $\{\mu_1, \ldots, \mu_K\}$ with $\mu_1 < \mu^* < \mu_K$:
\[
D = \sum_k \frac{1}{\mu_k} > \frac{K}{\mu^*}
\]
by convexity of $1/\mu$. Thus diversity increases $D$, helping satisfy the constraint.
\end{proof}

\section{Additional Experimental Details}
\label{app:experiments}

\subsection{Hyperparameter Settings}

\begin{table}[H]
\centering
\caption{Hyperparameters for all experiments.}
\label{tab:hyperparams}
\begin{tabular}{lcc}
\toprule
\textbf{Parameter} & \textbf{Supervised} & \textbf{PINN} \\
\midrule
Learning rate $\eta$ & $2 \times 10^{-3}$ & $2 \times 10^{-3}$ \\
Exponent LR multiplier & 1.0 & 0.02-0.1 \\
Training steps & 2,000 & 3,000-6,000 \\
Warmup steps & 0 & 500-1,500 \\
Exponent gradient clip $\delta$ & - & 0.03-0.1 \\
$K_e$ (even exponents) & 6 & 6 \\
$K_o$ (odd exponents) & 6 & 6 \\
$p_{\max}$ & 2.0-8.0 & 2.0-3.0 \\
$\beta_1$ (Müntz reg.) & $10^{-2}$ & $10^{-2}$ \\
$\beta_2$ ($L^1$ reg.) & $10^{-4}$ & $10^{-4}$ \\
$\lambda_{\text{BC}}$ & - & 200 (adaptive) \\
Collocation points & - & 2,048 \\
Boundary points & - & 256 \\
\bottomrule
\end{tabular}
\end{table}

\subsection{Computational Resources}

All experiments were conducted on a single NVIDIA A100 GPU. Supervised regression and singular ODE experiments complete in approximately 60-90 minutes total. The boundary-layer BVP experiments require 9-10 hours due to the increased training steps and multiple stiffness values. MSN training is slightly slower per iteration than MLP due to exp-log evaluation of $|x|^\mu$, but requires far fewer parameters for comparable accuracy.

%%%%%%%%%%%%%%%%%%%%%%%%%%%%%%%%%%%%%%%%%%%%%%%%%%%%%%%%%%%%%%%%%%%%%%%%%%%%%%%
% APPENDIX: ADDITIONAL EXPERIMENTAL RESULTS
%%%%%%%%%%%%%%%%%%%%%%%%%%%%%%%%%%%%%%%%%%%%%%%%%%%%%%%%%%%%%%%%%%%%%%%%%%%%%%%

\subsection{Additional Experimental Results}
\label{app:additional-results}

This section provides supplementary visualizations for the experiments presented in the main text.

\subsubsection{Supervised Regression}

\begin{figure}[H]
    \centering
    \begin{subfigure}[b]{0.48\textwidth}
        \centering
        \includegraphics[width=\textwidth]{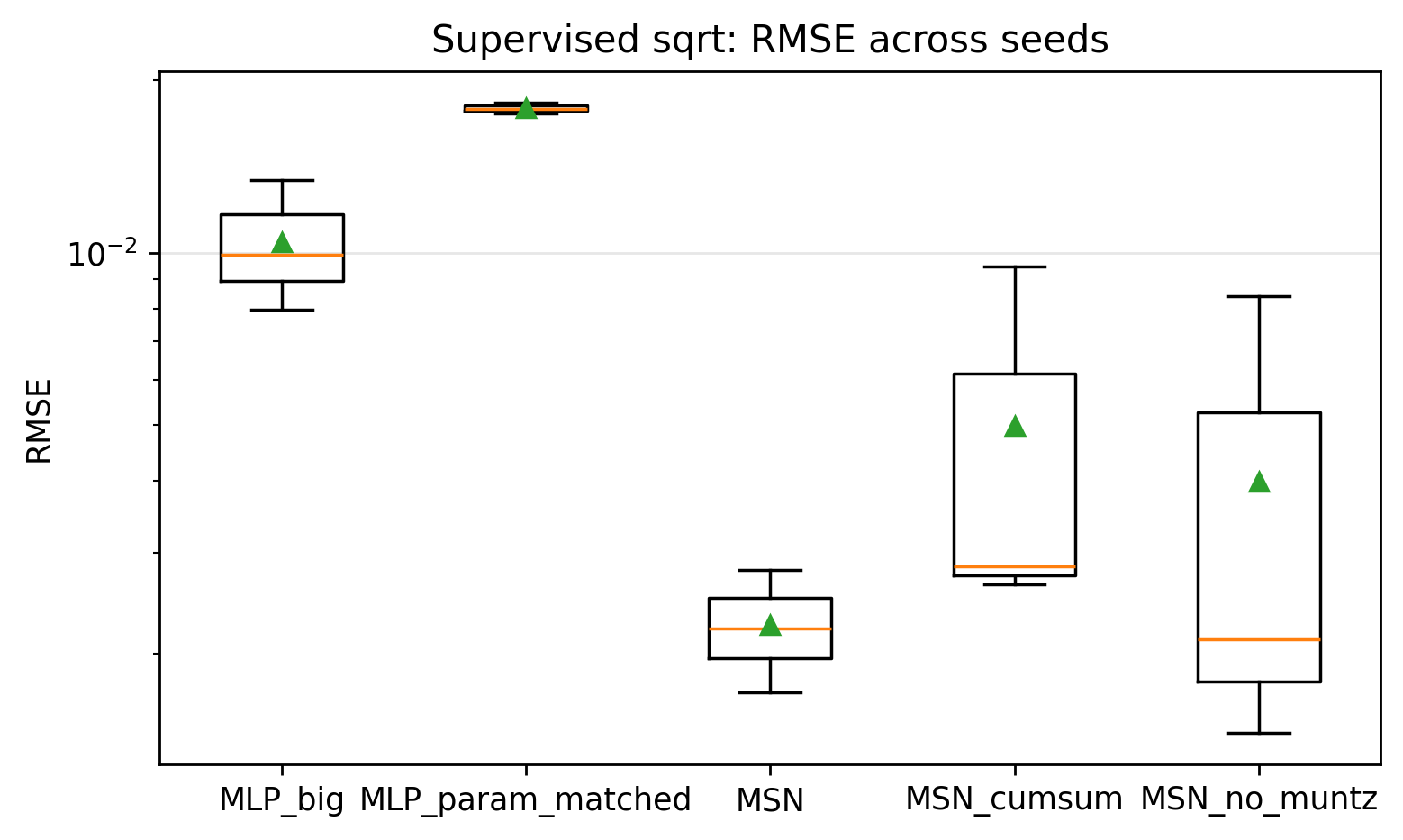}
        \caption{$\sqrt{x}$ task: RMSE distribution}
        \label{fig:app-sup-sqrt-box}
    \end{subfigure}
    \hfill
    \begin{subfigure}[b]{0.48\textwidth}
        \centering
        \includegraphics[width=\textwidth]{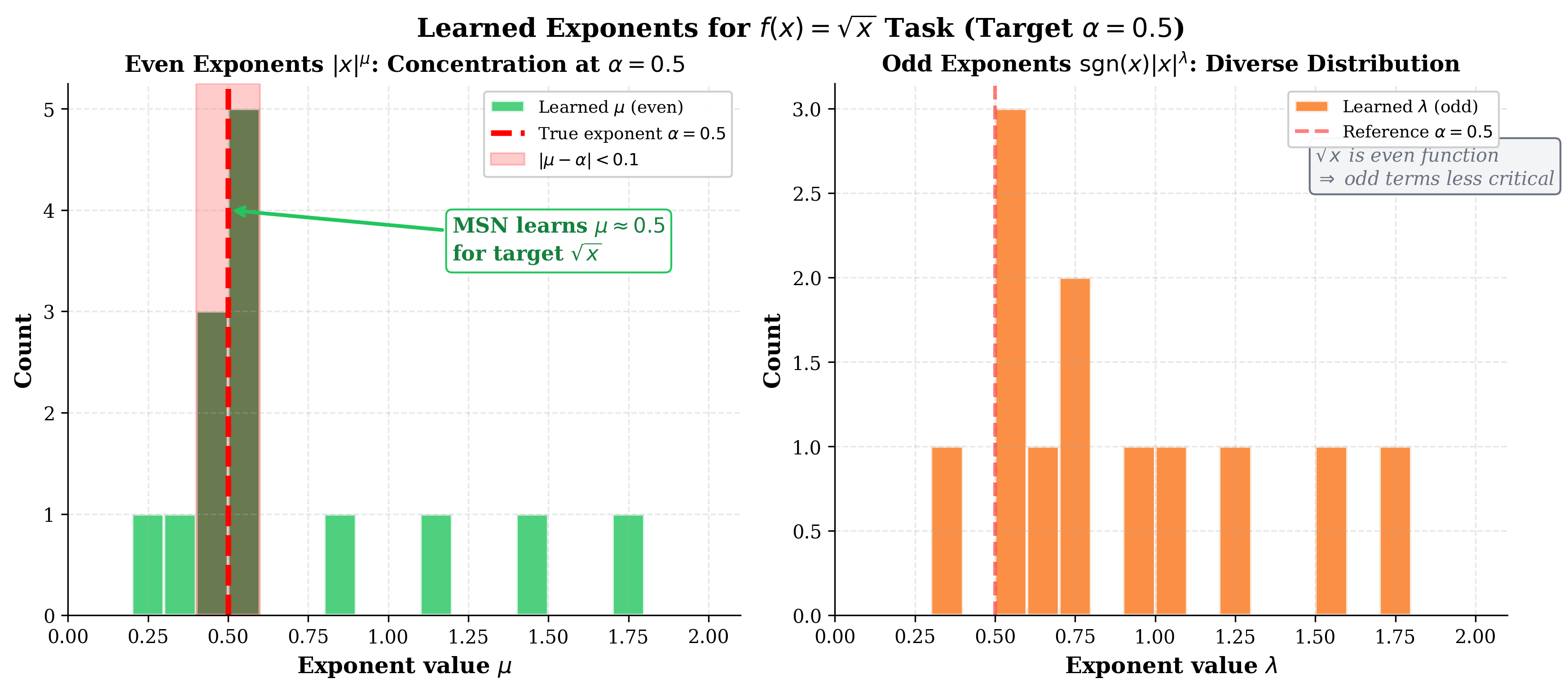}
        \caption{$\sqrt{x}$ task: Learned exponents (even and odd)}
        \label{fig:app-sup-sqrt-exp}
    \end{subfigure}
    
    \vspace{0.3cm}
    
    \begin{subfigure}[b]{0.48\textwidth}
        \centering
        \includegraphics[width=\textwidth]{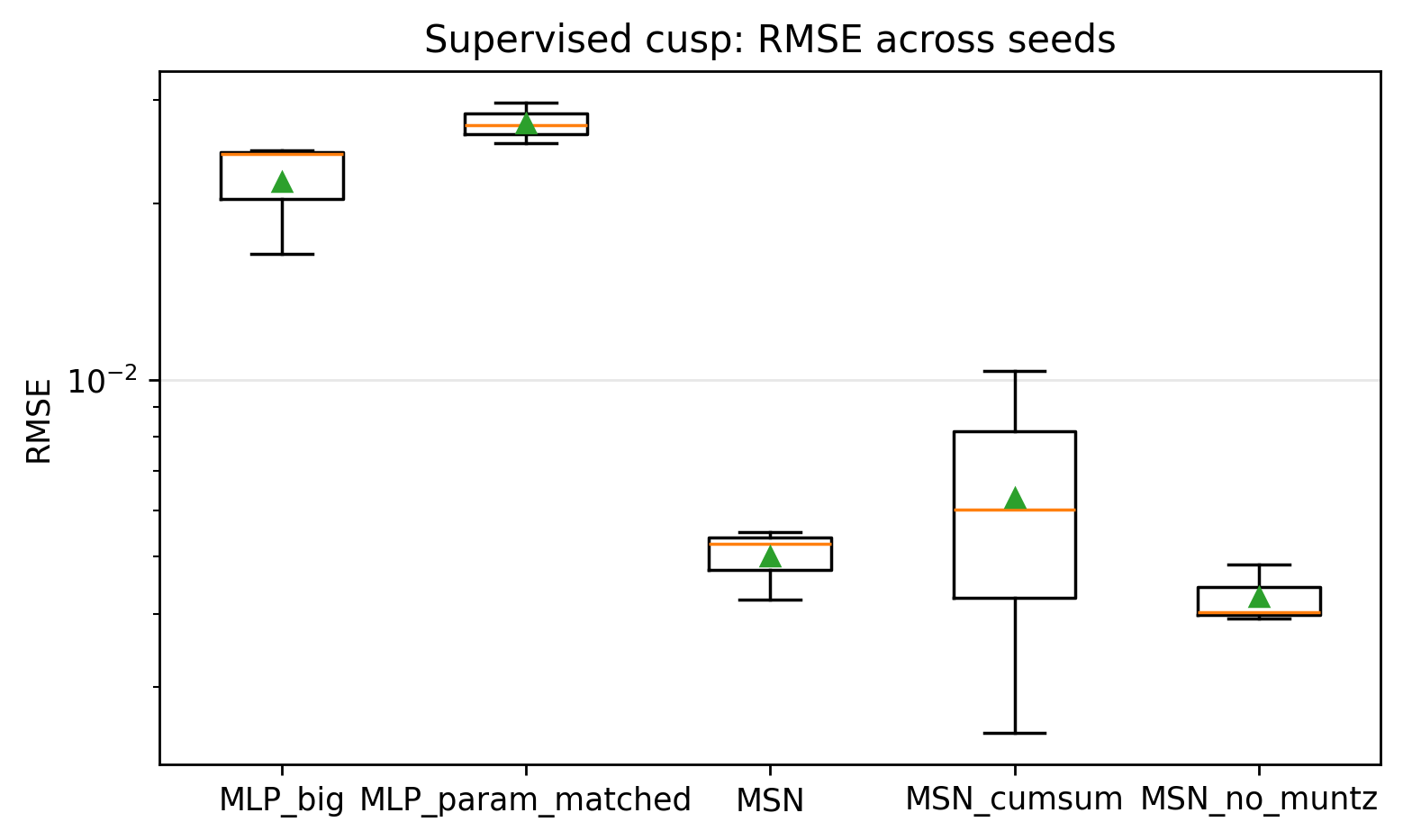}
        \caption{Cusp task: RMSE distribution}
        \label{fig:app-sup-cusp-box}
    \end{subfigure}
    \hfill
    \begin{subfigure}[b]{0.48\textwidth}
        \centering
        \includegraphics[width=\textwidth]{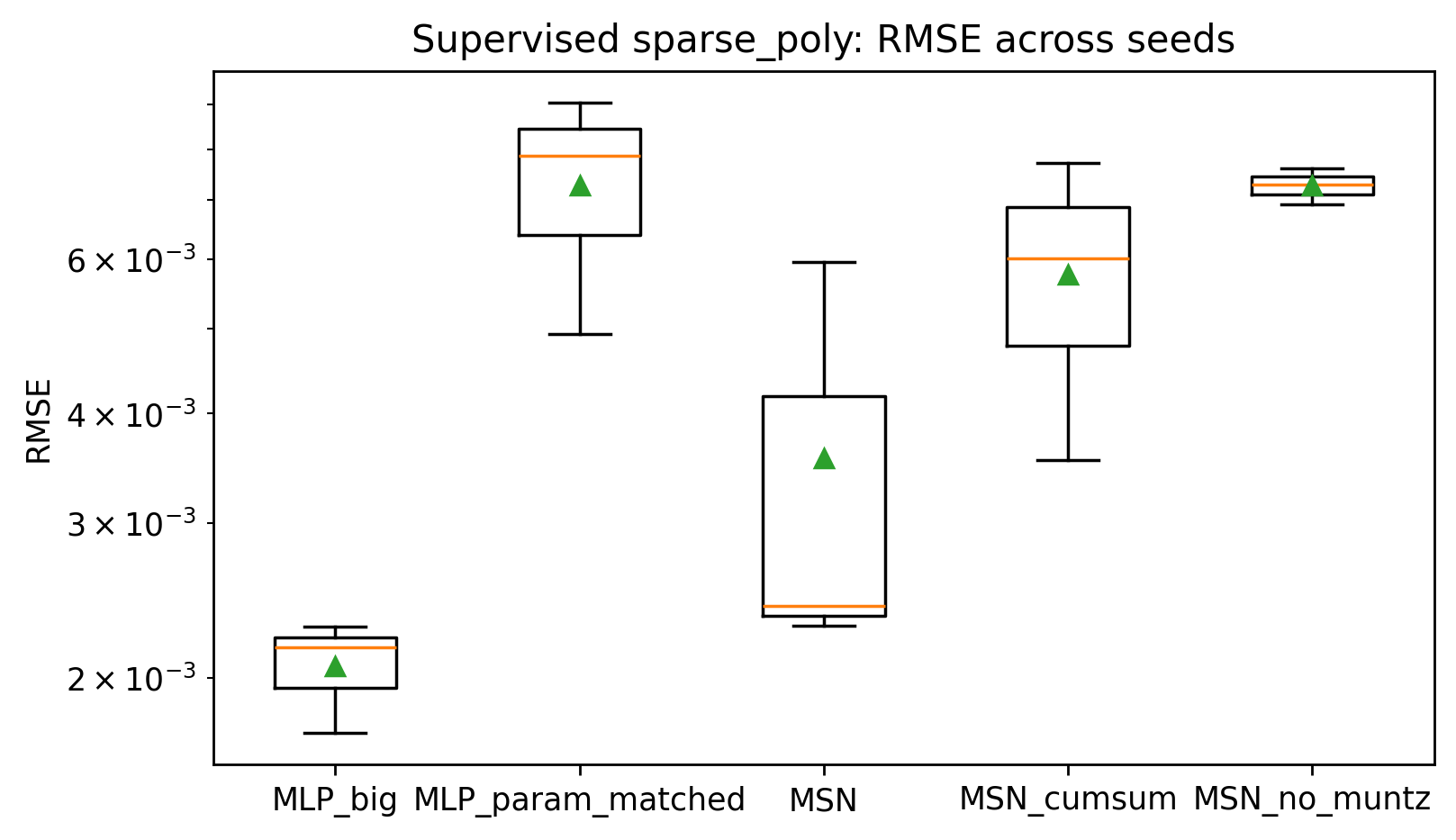}
        \caption{Sparse polynomial task: RMSE distribution}
        \label{fig:app-sup-poly-box}
    \end{subfigure}
    
    \caption{\textbf{Supervised regression: detailed results.} (a,c,d) RMSE distributions across 3 random seeds for each task, showing MSN's consistent advantage on singular functions ($\sqrt{x}$, cusp) and comparable performance on smooth functions (sparse polynomial). (b) Histogram of learned exponents for $\sqrt{x}$, showing concentration of even exponents near the true exponent $\mu = 0.5$, while odd exponents remain diverse (as expected for the even function $\sqrt{x}$).}
    \label{fig:app-supervised}
\end{figure}

\begin{figure}[H]
    \centering
    \includegraphics[width=\textwidth]{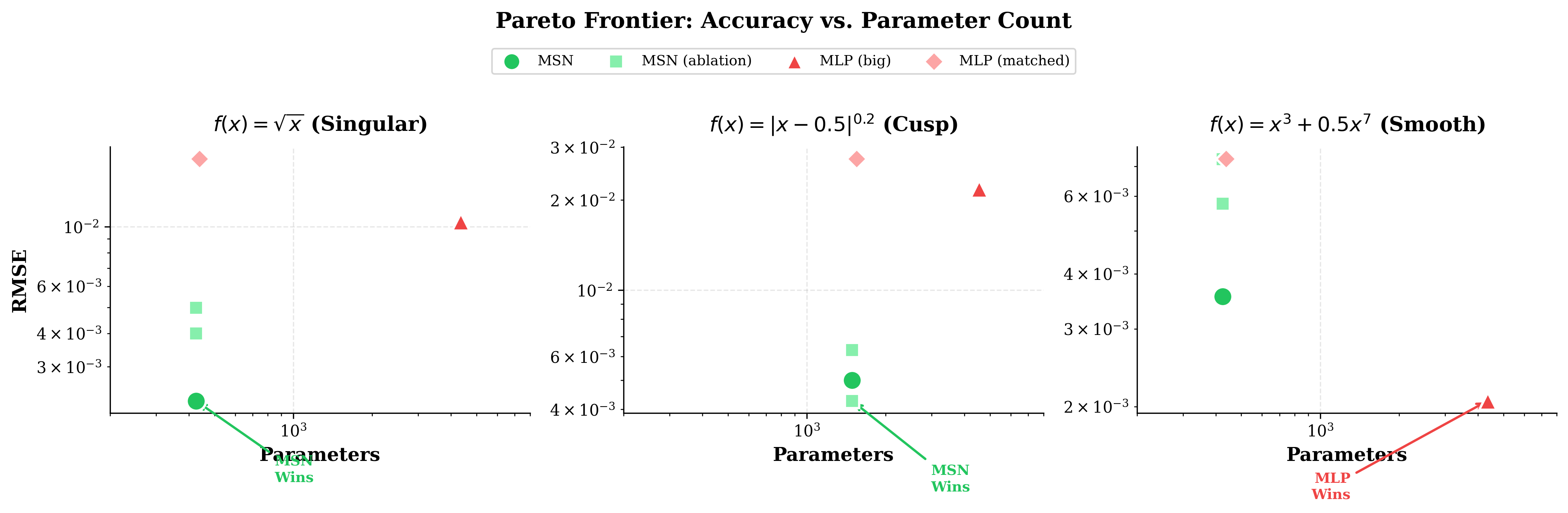}
    \caption{\textbf{Pareto frontier: accuracy vs.\ parameter count.} MSN achieves the best trade-off for singular functions ($\sqrt{x}$, cusp), reaching lower error with fewer parameters than both MLP variants. For the smooth polynomial task, MLP (big) achieves the best performance, confirming that MSN's advantage is specific to singular structure. Each subplot shows a different task with clear winner annotations.}
    \label{fig:app-pareto}
\end{figure}

\subsubsection{PINN: Singular ODE}

\begin{figure}[H]
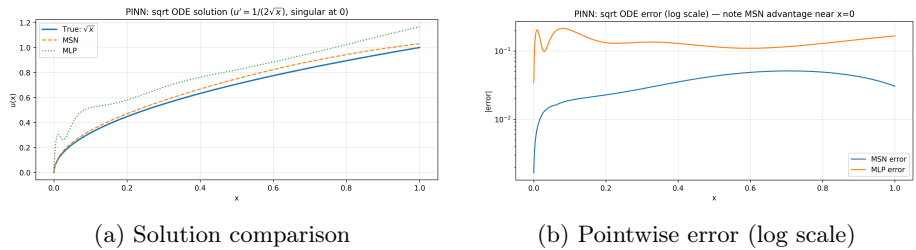

    \centering
    \begin{subfigure}[b]{0.48\textwidth}
        \centering
        \includegraphics[width=\textwidth]{figures/pinn_sqrt_solution.png}
        \caption{Solution comparison}
        \label{fig:app-sqrt-sol}
    \end{subfigure}
    \hfill
    \begin{subfigure}[b]{0.48\textwidth}
        \centering
        \includegraphics[width=\textwidth]{figures/pinn_sqrt_error.png}
        \caption{Pointwise error (log scale)}
        \label{fig:app-sqrt-err}
    \end{subfigure}
    
    \caption{\textbf{Singular ODE $u' = 1/(2\sqrt{x})$: solution quality.} (a) MSN closely tracks the true $\sqrt{x}$ solution, while MLP exhibits visible deviation. (b) Pointwise error on log scale shows MSN achieves lower error across the domain, particularly near the singularity at $x = 0$.}
    \label{fig:app-sqrt-solution}
\end{figure}

\begin{figure}[H]
    \centering
    \begin{subfigure}[b]{0.48\textwidth}
        \centering
        \includegraphics[width=\textwidth]{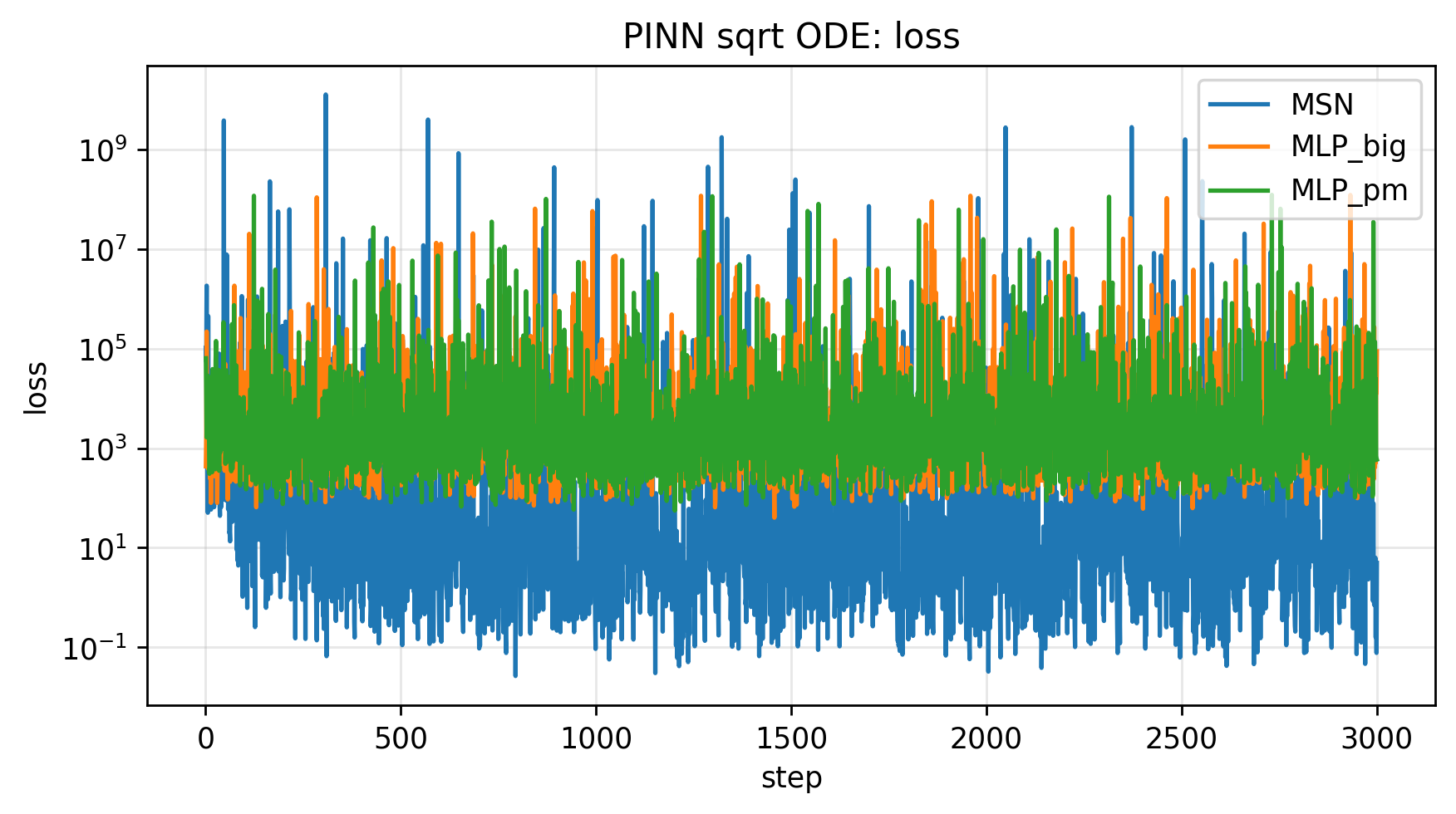}
        \caption{Total loss during training}
        \label{fig:app-sqrt-loss}
    \end{subfigure}
    \hfill
    \begin{subfigure}[b]{0.48\textwidth}
        \centering
        \includegraphics[width=\textwidth]{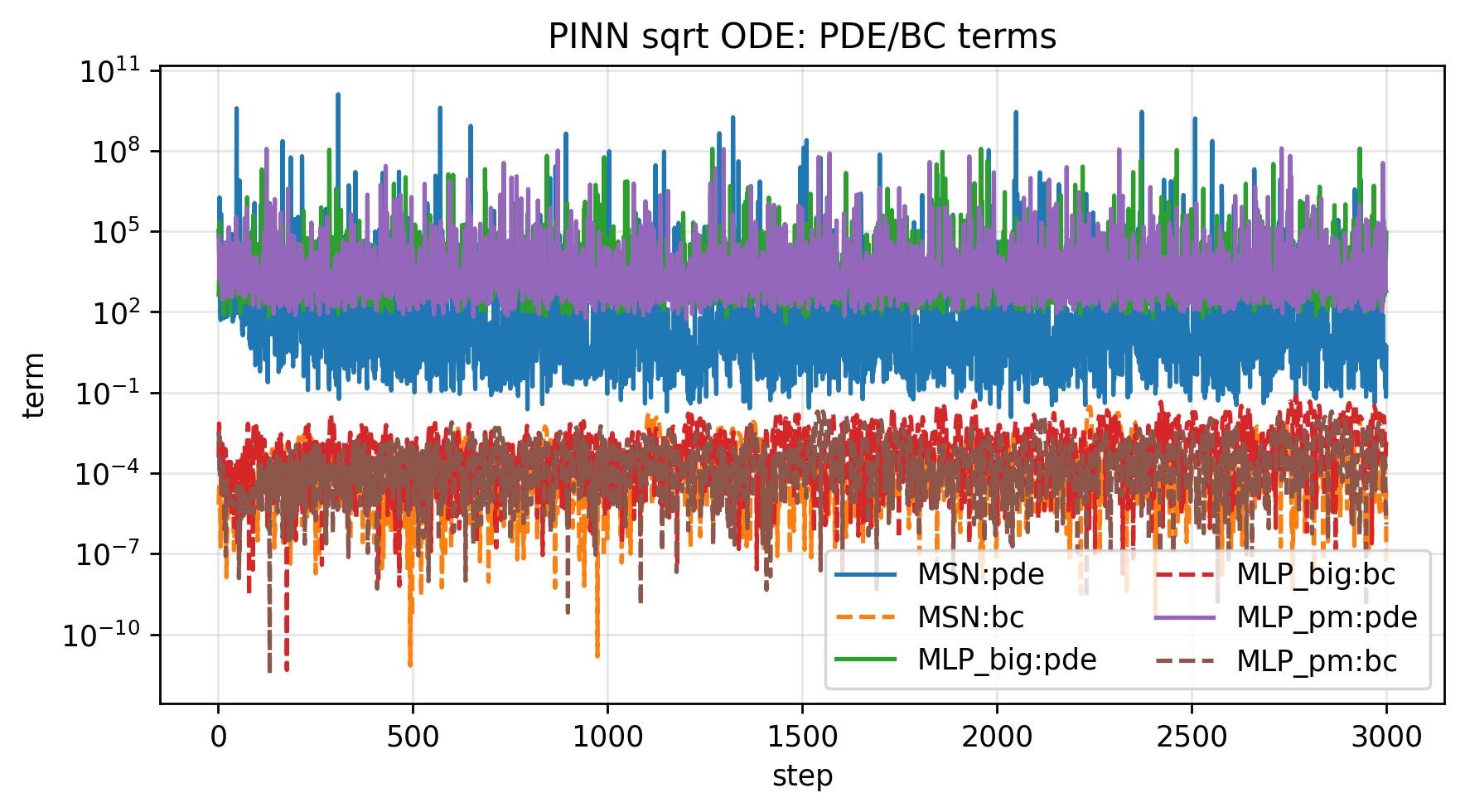}
        \caption{PDE and BC loss components}
        \label{fig:app-sqrt-terms}
    \end{subfigure}
    
    \vspace{0.3cm}
    
    \begin{subfigure}[b]{0.48\textwidth}
        \centering
        \includegraphics[width=\textwidth]{figures/pinn_sqrt_exponent_hist.png}
        \caption{Learned exponent distribution}
        \label{fig:app-sqrt-exp}
    \end{subfigure}
    \hfill
    \begin{subfigure}[b]{0.48\textwidth}
        \centering
        \includegraphics[width=\textwidth]{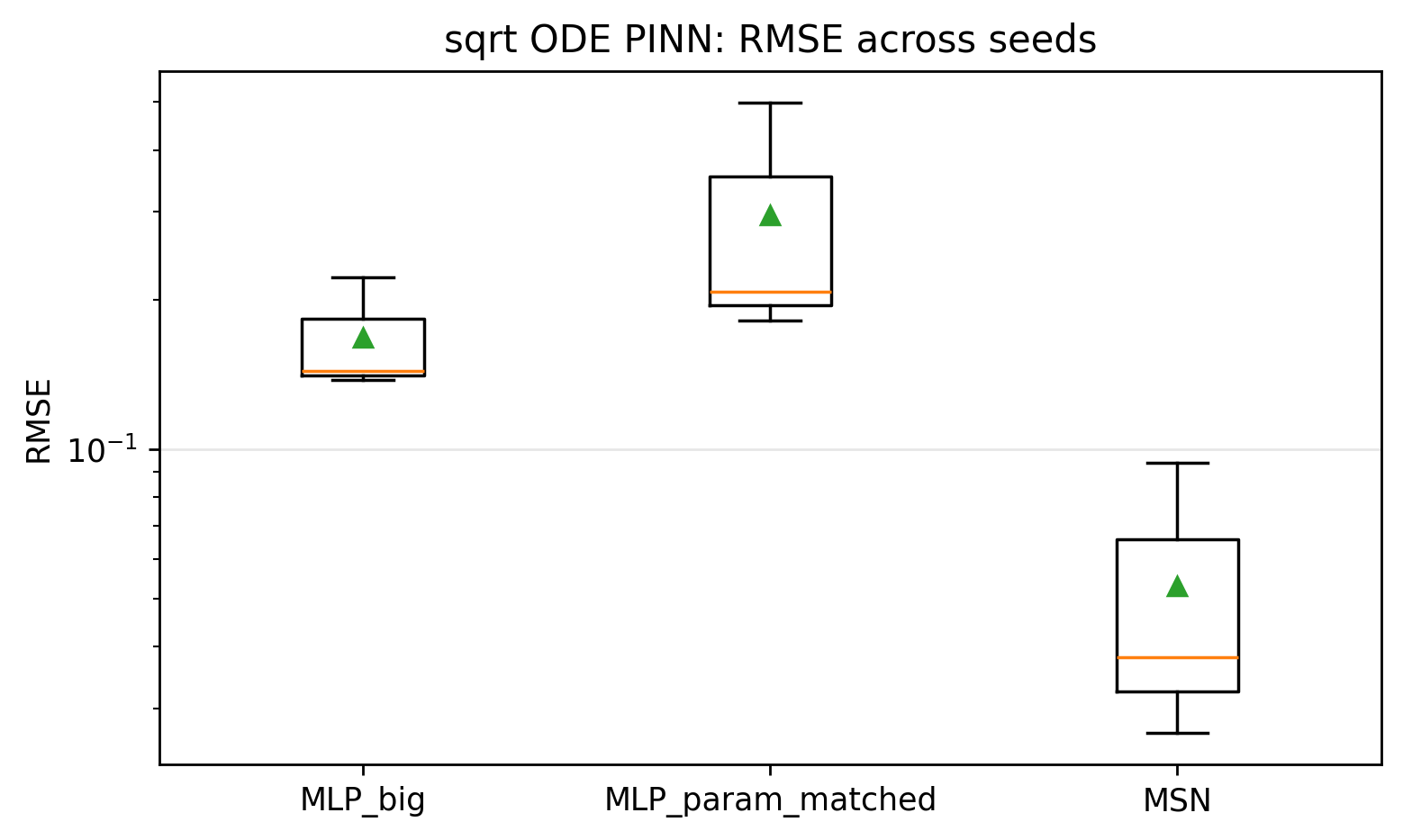}
        \caption{RMSE across random seeds}
        \label{fig:app-sqrt-box}
    \end{subfigure}
    
    \caption{\textbf{Singular ODE: training dynamics.} (a) MSN converges to lower total loss than MLP. (b) Both PDE residual and boundary condition losses decrease faster for MSN. (c) Learned exponents concentrate near $\mu \approx 0.5$, matching the true solution structure. (d) RMSE distribution across 3 seeds confirms consistent MSN advantage.}
    \label{fig:app-sqrt-training}
\end{figure}

\subsubsection{PINN: Boundary-Layer BVP}

\begin{figure}[H]
    \centering
    \begin{subfigure}[b]{0.48\textwidth}
        \centering
        \includegraphics[width=\textwidth]{figures/bl_eps0.05_solution.png}
        \caption{Solution at $\epsilon = 0.05$}
        \label{fig:app-bl-sol}
    \end{subfigure}
    \hfill
    \begin{subfigure}[b]{0.48\textwidth}
        \centering
        \includegraphics[width=\textwidth]{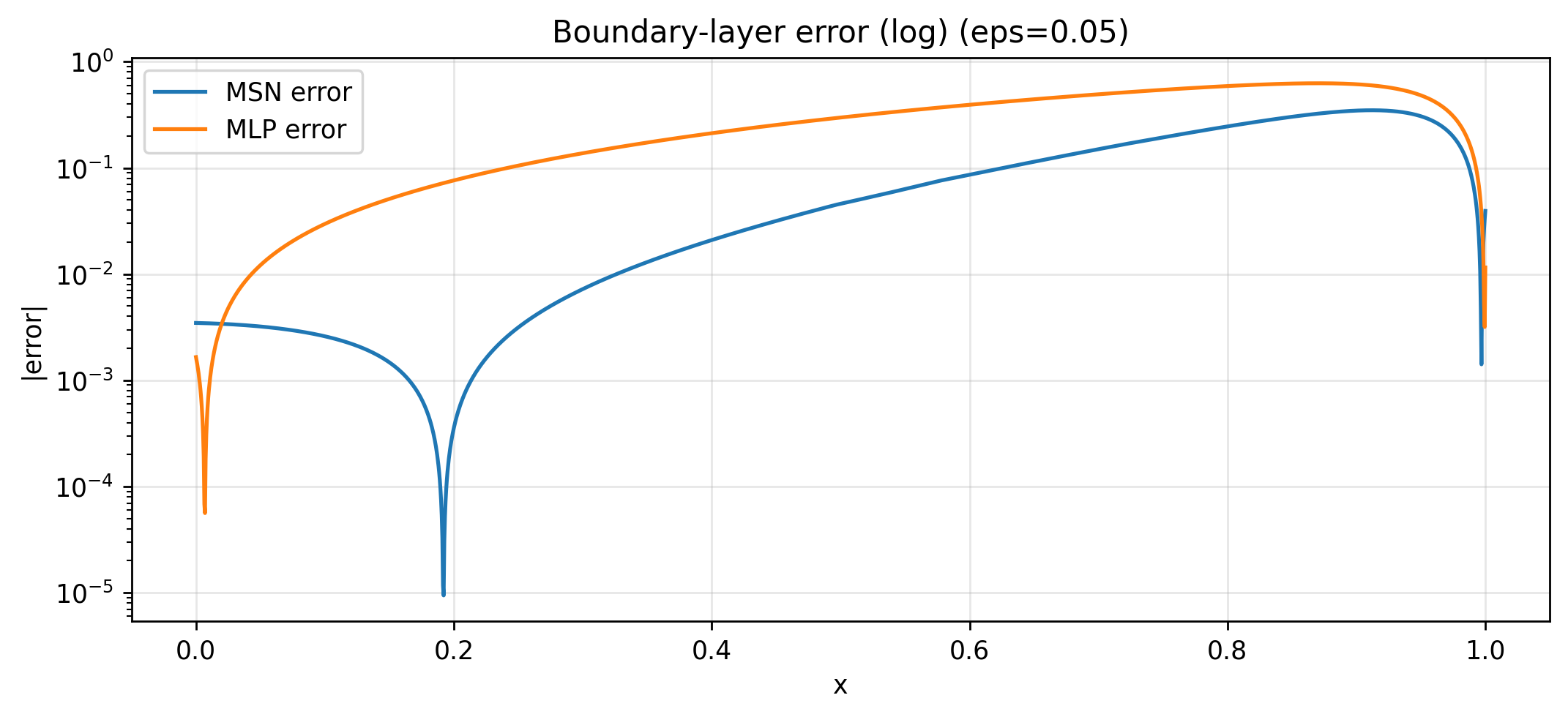}
        \caption{Pointwise error at $\epsilon = 0.05$}
        \label{fig:app-bl-err}
    \end{subfigure}
    
    \caption{\textbf{Boundary-layer BVP: solution quality.} (a) MSN captures the sharp boundary layer near $x = 1$ more accurately than MLP. (b) Error is concentrated in the boundary layer region for both methods, but MSN achieves lower peak error.}
    \label{fig:app-bl-solution}
\end{figure}

\begin{figure}[H]
    \centering
    \begin{subfigure}[b]{0.48\textwidth}
        \centering
        \includegraphics[width=\textwidth]{figures/bl_eps0.05_exponent_traj.png}
        \caption{Exponent trajectories during training}
        \label{fig:app-bl-traj}
    \end{subfigure}
    \hfill
    \begin{subfigure}[b]{0.48\textwidth}
        \centering
        \includegraphics[width=\textwidth]{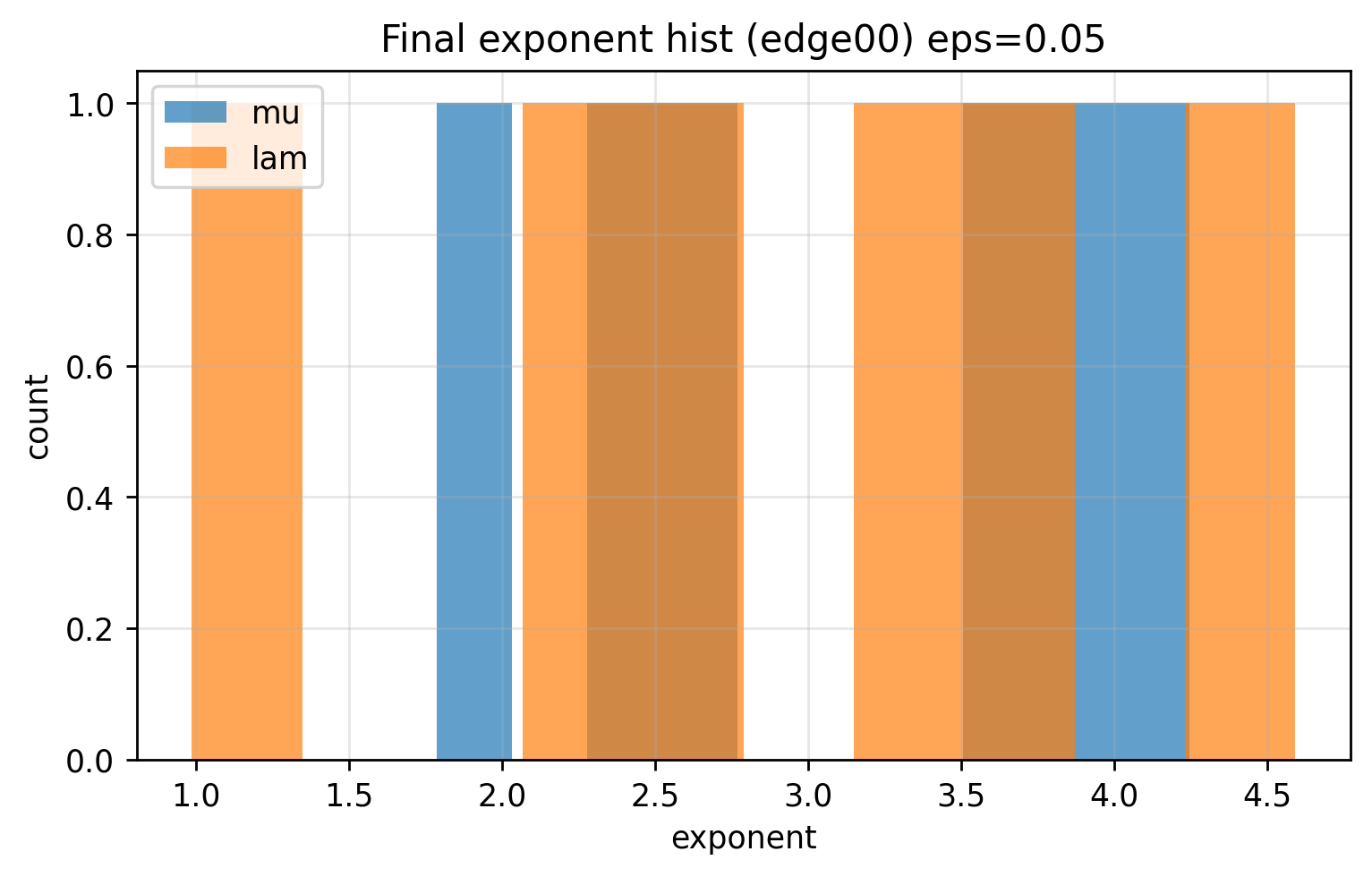}
        \caption{Final exponent distribution}
        \label{fig:app-bl-hist}
    \end{subfigure}
    
    \caption{\textbf{Boundary-layer BVP: exponent dynamics.} (a) Exponents evolve during training, initially exploring before converging to stable values. The warmup period (first 500 steps) freezes exponents to allow coefficient adaptation. (b) Final exponent distribution shows diversity, with some exponents adapting to capture the boundary layer structure.}
    \label{fig:app-bl-exponents}
\end{figure}

\subsubsection{Gradient Diagnostics}

\begin{figure}[H]
    \centering
    \includegraphics[width=0.7\textwidth]{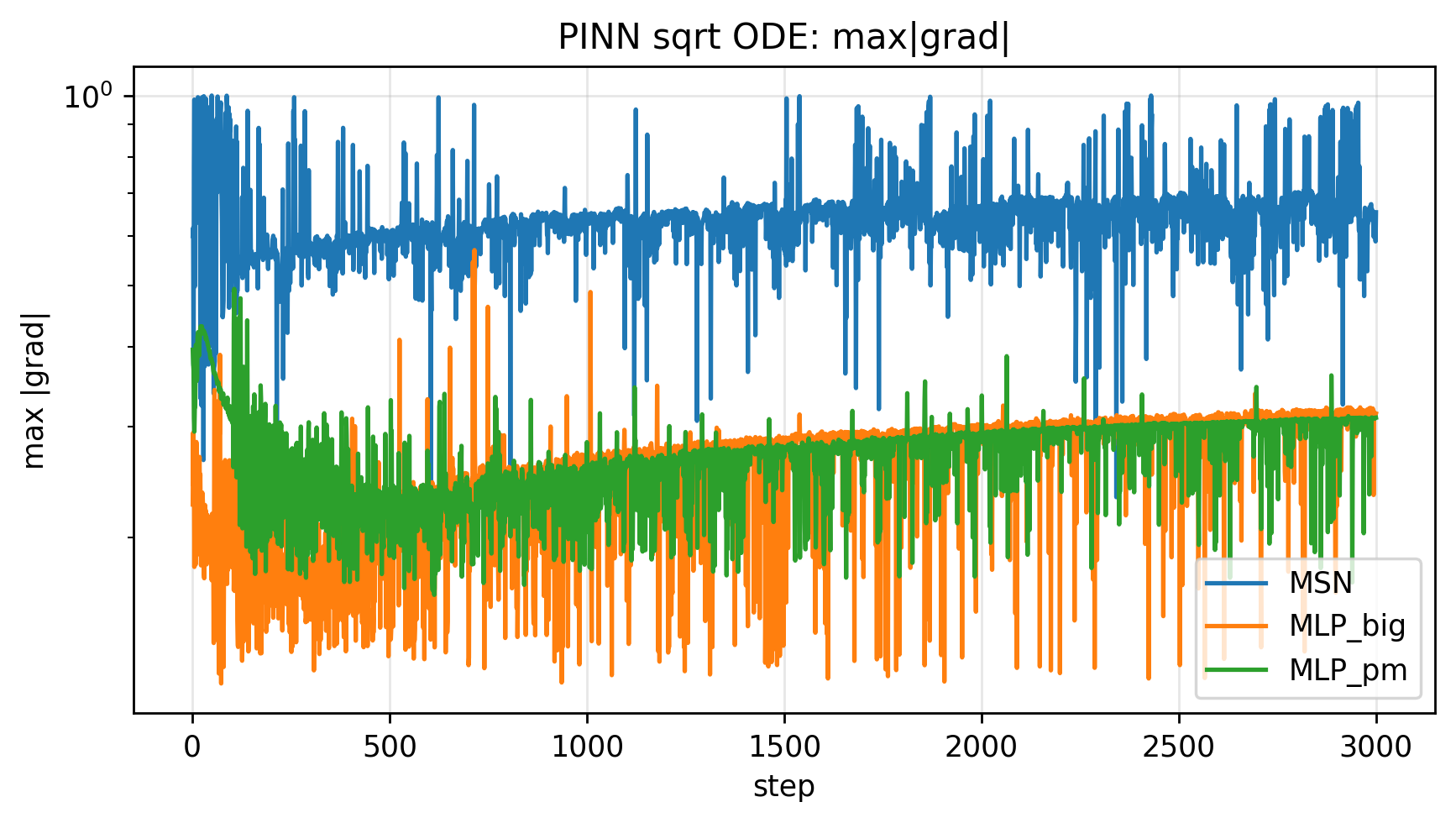}
    \caption{\textbf{Gradient diagnostics for singular ODE.} Gradient norms for coefficient parameters and exponent parameters during training. The two-time-scale optimization (exponent learning rate = $0.02\times$ coefficient learning rate) and gradient clipping maintain stable training despite the challenging singular structure.}
    \label{fig:app-grad}
\end{figure}

\end{document}